\documentclass{article}

\usepackage[preprint]{neurips_2026} 
\usepackage[utf8]{inputenc} 
\usepackage[T1]{fontenc}    
\usepackage{hyperref}       
\usepackage{url}            
\usepackage{amssymb}
\usepackage{algorithm}      
\usepackage{algpseudocode}
\usepackage{booktabs}       
\usepackage{amsfonts}       
\usepackage{nicefrac}       
\usepackage{microtype}      
\usepackage{xcolor}         

\usepackage{float}
\usepackage{amsmath} 
\usepackage{amsthm}
\usepackage{multirow}
\usepackage{graphicx}
\usepackage{pifont}
\usepackage{subcaption}
\usepackage{ragged2e}
\usepackage{arydshln}
\usepackage{colortbl}
\usepackage{adjustbox}

\definecolor{bestblue}{HTML}{A3C1DA} 

\newcommand{\first}[1]{\cellcolor{bestblue!60}\textcolor{black}{#1}}

\newcommand{\second}[1]{\cellcolor{bestblue!25}\textcolor{black}{#1}}

\graphicspath{ {./images/} }

\newcommand{\cmark}{\ding{51}}%
\newcommand{\xmark}{\ding{55}}%

\newcommand{\cN}{\mathcal{N}}
\newcommand{\cB}{\mathcal{B}}

\newcommand{\D}{\mathcal{D}}
\newcommand{\Dphy}{\mathcal{D}_{\mathrm{phy}}}
\newcommand{\Dic}{\mathcal{D}_{\mathrm{ic}}}
\newcommand{\Dbc}{\mathcal{D}_{\mathrm{bc}}}
\newcommand{\Dlab}{\mathcal{D}_{\mathrm{label}}}

\newcommand{\cL}{\mathcal{L}}
\newcommand{\Lphy}{\mathcal{L}_{\mathrm{phy}}}
\newcommand{\Lic}{\mathcal{L}_{\mathrm{ic}}}
\newcommand{\Lbc}{\mathcal{L}_{\mathrm{bc}}}
\newcommand{\Llab}{\mathcal{L}_{\mathrm{label}}}
\newcommand{\Ldata}{\mathcal{L}_{\mathrm{data}}}

\newcommand{\lphy}{\ell_{\mathrm{phy}}}
\newcommand{\lic}{\ell_{\mathrm{ic}}}
\newcommand{\lbc}{\ell_{\mathrm{bc}}}
\newcommand{\llab}{\ell_{\mathrm{label}}}
\newcommand{\lab}{\mathrm{label}}
\newcommand{\phy}{\mathrm{phy}}
\newcommand{\ic}{\mathrm{ic}}
\newcommand{\bc}{\mathrm{bc}}

\newcommand{\wphy}{w_{\mathrm{phy}}}
\newcommand{\wic}{w_{\mathrm{ic}}}
\newcommand{\wbc}{w_{\mathrm{bc}}}
\newcommand{\wdata}{w_{\mathrm{data}}}
\newcommand{\diag}{\operatorname{diag}}

\newtheorem{remark}{Remark}
\newtheorem{theorem}{Theorem}
\newtheorem{lemma}{Lemma}

\newtheorem{assumption}{Assumption}

\allowdisplaybreaks
\title{AdamFLIP: Adaptive Momentum Feedback Linearization Optimization for Hard Constrained PINN Training
}

\author{
  Binghang Lu\textsuperscript{1}\thanks{Equal contribution.} \quad
  Runyu Zhang\textsuperscript{2}\footnotemark[1] \quad
  Changhong Mou\textsuperscript{3} \quad
  Na Li\textsuperscript{4} \quad
  Guang Lin\textsuperscript{1} \\
  \textsuperscript{1} Purdue University \\
  \textsuperscript{2} Massachusetts Institute of Technology \\
  \textsuperscript{3} Utah State University \\
  \textsuperscript{4} Harvard University
}
\allowdisplaybreaks
\begin{document}

\maketitle

\begin{abstract}

Physics-informed neural networks (PINNs) provide a flexible framework for solving forward and inverse problems governed by partial differential equations (PDEs), but standard PINN training typically relies on soft penalty formulations that combine PDE residuals, data mismatch, and initial/boundary conditions using manually chosen weights. This often leads to ill-conditioning, sensitivity to loss weights, and poor constraint satisfaction. In this work, we reformulate PINN training as an equality-constrained optimization problem and propose a novel Adaptive Momentum Feedback Linearization Optimization for Hard Constrained PINN (AdamFLIP). The key idea is to view the constraint residuals as the output of a controlled dynamical system and to compute the Lagrange multiplier as a feedback input that locally drives these residuals toward stable linear contraction dynamics. AdamFLIP then applies Adam-style first- and second-moment adaptation to the resulting feedback-linearized Lagrangian gradient, combining principled constraint handling with the scalability and robustness of adaptive neural-network optimization.  
We test AdamFLIP on a range of benchmark forward and inverse PDE problem, and
it consistently outperforms both the standard soft-constrained PINN and state-of-the-art constrained optimizers.
Specifically, on the Navier--Stokes equations benchmark, AdamFLIP \textbf{reduces relative $L_2$ error by more than two thirds} for the predicted solution compared to the next best method. 
Our AdamFLIP framework  provides an effective and computationally scalable hard constraint optimization method for PINN training.

\end{abstract}

\section{Introduction}


Physics informed neural networks (PINNs) have become 
more and more popular in scientific machine learning which helps bridge the gap between data-driven modeling and classical PDE-based simulation. In a standard PINN (c.f. \cite{raissi2019physics,karniadakis2021physics,cuomo2022scientific}), one represents the solution of a PDE with a neural network and trains it by minimizing a composite objective loss function that includes the residual of the governing differential equation (via automatic differentiation of the network) together with initial/boundary conditions and any available data. The learned solution is therefore expected to be approximate both the available data and the governing equations. This framework can naturally extend to both forward and inverse problems; in the inverse problem setting, unknown PDE parameters are treated as additional trainable variables and inferred jointly with the solution during optimization.
PINNs have consequently been applied in a wide range of engineering and scientific fields \cite{cai2021physics,guo2025advances}. 

Most existing PINN training approaches use a soft-penalty formulation, in which the PDE residual, boundary conditions (BCs), and initial conditions (ICs) are summed as separately weighted terms in a composite loss function. 
However, improper weight selection can lead to ill-conditioning, slow convergence, or solutions that fit the data while significantly violating the governing equations or boundary and initial conditions.
Moreover, because each loss term may differ in scale, minimizing the composite loss does not guarantee that each term is individually minimized.
Hence, the PINN's performance is highly sensitive to choice of penalty weights  which has been shown in extensive empirical and theoretical studies \cite{wang2021understanding,krishnapriyan2021characterizing,wang2023expert,maddu2022inverse}.


These challenges have motivated several recent methods that incorporate constrained-optimization ideas into PINN training while retaining the soft-penalty objective. 
For example, Xu and Darve \cite{xu2021trust} adopt a trust-region strategy, and ADMM-PINNs \cite{song24admmpinn} use operator splitting to decouple training. However, both methods still remain penalty weights dependent.
A more fundamental shift recasts PINN training entirely as a constrained optimization problem. One prominent direction uses augmented Lagrangian (AL) methods \cite{lu2021physics_hard_constraints,son2023enhanced,hu2025conditionally}, which treats ICs and BCs as equality constraints and performs primal updates coupled with multiplier (and penalty) updates to adaptively emphasize violated constraints. 
A second direction leverages sequential quadratic programming (SQP) and trust-region formulations;  for example,
\cite{cheng2024physics} use sequential quadratic programming with trust-region to mitigate ill-conditioning and instability observed in standard PINNs (trSQP-PINN). 
Rather than enforcing constraints through the optimizer, a third class of methods builds them directly into the network architecture. Early work in this direction embedded constraints into the trial solution \cite{lagaris1998artificial,djeumou2022neural}. 
More recent approaches hard-impose boundary conditions through Fourier-feature parameterizations \cite{straub2025hard,LI202460} or add a projection layer derived from Karush–Kuhn–Tucker (KKT) conditions to exactly satisfy linear equality constraints \cite{chen2024physics}.


Beyond PINNs, constrained optimization is a long-standing theme in machine learning and reinforcement learning \cite{kotary2021end,garcia2015comprehensive}, where the goal is to optimize a learning objective while satisfying explicit constraints subject to explicit operational or physical constraints. 
Standard approaches in this broader literature include primal–dual Lagrangian methods (e.g.\ \cite{nandwani2019primal,stooke2020responsive}), SQP-style methods (e.g.\ \cite{berahas2021sequential}), and hard-constraint-by-construction approaches (e.g.\ \cite{min2024hardnet,grontas2025pinet}). 
Recently, a separate line of work \cite{zhang2025constrained,zhang2025zeroth,cerone2025new} has brought a control-theoretic perspective to constrained optimization that proposes a feedback linearization (FL) solver that imports a classical nonlinear control idea into constrained learning. Rather than relying on primal-dual updates or second-order subproblems, FL designs the multiplier/feedback update to impose a stable, approximately linear contraction on constraint violations, and has been developed with rigorous convergence theories for general constrained nonconvex problems.

\paragraph{Our Contributions.} 
We bring this FL framework to PINN training and make the following contributions. 
First, we formulate PINN training as an equality-constrained problem that explicitly separates feasibility (IC/BC/PDE satisfaction) from optimality, and design the multiplier/feedback update to shape the constraint-violation dynamics into a stable linear system. Second, we introduce AdamFLIP, an adaptive variant that integrates FL with moment estimation following Adam \cite{kingma2014adam}. 
AdamFLIP preserves the fast convergence and robustness of adaptive gradient methods while incorporating principled constraint handling through FL. Empirically, AdamFLIP consistently outperforms both standard soft-penalty PINN training with Adam and representative constrained PINN optimizers, achieving significantly improved constraint satisfaction and solution accuracy for both forward and inverse problems. 
Third, we provide finite-time convergence analysis for a metric-compatible variant of AdamFLIP, showing that both the stationarity residual and constraint violation achives best-iterate converge at a rate of $\mathcal{O}(\log{T}/\sqrt{T})$ under standard assumptions.
Together, these results demonstrate that embedding feedback-linearization principles into modern adaptive optimizers offers a simple yet powerful alternative to penalty reweighting and classical constrained-training schemes.

Due to space constraints, we defer a more comprehensive review of the related literature and a detailed comparison with our work to Appendix~\ref{sec:related-works}.
\section{Preliminaries: Physics Informed Neural Networks (PINNs) \label{sec:pinn}}
To illustrate the framework of \textit{physics informed neural networks} (PINNs) \cite{cai2021physics,raissi2019physics,karniadakis2021physics}, we start with the general nonlinear PDE which takes the following form:
\begin{equation}\label{eq:PDE}
u_t + \cN[u] = 0,\ x \in \Omega, \ t\in[0,T],
\end{equation}
where $\Omega \subset \mathbb{R}^D$ is an open, bounded $D$ dimensional domain, 
$u : [0,T] \times \Omega \to \mathbb{R}$ denotes the unknown solution field, and 
$\cN[\cdot]$ is a (generally nonlinear) differential operator acting on~$u$.

Equation~\eqref{eq:PDE} is supplemented with the initial condition (IC):
\begin{equation}\label{eq:PDE-ic}
    u(0,x) = u_0(x), \qquad x \in \Omega,
\end{equation}
and suitable boundary conditions (BCs) prescribed on the spatial boundary $\partial\Omega$. Without loss of generality, we denote these conditions abstractly by
\begin{equation}\label{eq:PDE-bc}
\cB[u](t,x) = h(t,x), 
\qquad (t,x) \in [0,T] \times \partial\Omega,
\end{equation}
where $\cB$ is a boundary operator (e.g., Dirichlet, Neumann, Robin, or a combination) and $h$ is a given function.

PINNs offer a unified framework for combining observational data with physical laws. The key idea is to incorporate the governing PDE~\eqref{eq:PDE} together with the associated initial and boundary conditions \eqref{eq:PDE-bc}) directly into the loss function of a neural network where the required derivatives are obtained via automatic differentiation. In this setting, the solution $u$ is approximated by a neural network which is trained both to fit available data and to minimize the residuals of the PDE and its boundary conditions. Consequently, the learned surrogate remains consistent with the underlying physics. More specifically, the PINN framework introduces the PDE residual loss (physics loss) denoted as $\lphy$ , initial condition loss denoted as $\lic$ and boundary condition loss denoted as $\lbc$ according to \eqref{eq:PDE}, \eqref{eq:PDE-ic}, \eqref{eq:PDE-bc} as follows: 
\vspace{-10pt}
\begin{align}
   \lphy(u;t,x) &:= u_t(t,x) + \cN[u](t,x);\label{eq:residual}\\
   \lic(u;t,x) &:= (u(0,x) - u_0(x))^2;\label{eq:initial}\\
\lbc(u;t,x) &:= (\cB[u](t,x) -h(t,x))^2 \label{eq:boundary}
\end{align}

In specific, in PINN we collect samples (generally from a mesh) $\mathcal{D}:=\{(t_i,x_i)\}$ and we define
\begin{equation}
\begin{split}
\Dphy&:=\{(x,t)\in \D: x\in\Omega, \ t\in[0,T]\},\\~~  \Dic&:=\{(x,t)\in \D: x\in\Omega, \ t=0\}, \\\Dbc&:=\{(x,t)\in \D: (t,x) \in [0,T] \times \partial\Omega\}
\end{split}
\end{equation}
Then assume that $u$ is parameterized by the parameter of a neural network, i.e. $u = u_\theta$, then the total PDE residual loss $\Lphy$, boundary loss $\Lbc$ and initial condition $\Lic$ loss with respect to the parameter $\theta$ are given by
\begin{align*}
    \textstyle \Llab(\theta):=\frac{1}{|\Dlab|}\sum_{(t_i, x_i)\in\Dlab} \llab(u_\theta; t_i,x_i); ~~\lab \in \{\phy, \ic, \bc\}
\end{align*}

\paragraph{Forward Problem} For the forward problem, we assume that the differential operator $\cN$ is fully known and the task is to find a proper $u_\theta$ that makes $\Lphy, \Lic, \Lbc$ as small as possible, and in general, the classical approach solves this by solving the unconstrained optimization problem 
\begin{equation}\label{eq:forward}
    \min_\theta \cL(\theta):= \wphy\Lphy(\theta) +\wic \Lic(\theta) + \wbc\Lbc(\theta),
\end{equation}
where $\wphy, \wic, \wbc$ are some hand-tuned weight on different losses. Then we solve this by first-order methods such as gradient descent or Adam~\cite{kingma2014adam}.

\paragraph{Inverse Problem} The inverse problem considers a more challenging setting in which the differential operator $\cN$ in \eqref{eq:PDE} depends on unknown parameters $\kappa$ (e.g. the unknown viscosity term $v$ in 1D burgers equation). The goal is to jointly estimate the true parameter $\kappa$ while learning an accurate solution $u_\theta$ from noisy observational data. To clarify the notation, we explicitly denote the parameterized differential operator by $\cN_\kappa$. Accordingly, the associated physics loss is redefined as
\begin{align*}
   \textstyle \lphy(u,\kappa;t,x) := u_t(t,x) + \cN_\kappa[u](t,x), \quad \Lphy(\theta,\kappa):=\frac{1}{|\Dphy|}\sum_{(t_i, x_i)\in\Dphy} \lphy(u_\theta, \kappa; t_i,x_i)
\end{align*}
whereas the definitions of $\Lic,\Lbc$ still remain unchanged. Note that the goal is to both recover the optimal parameter of the neural network $\theta$ and the parameter of the differential operator $\kappa$. To achieve this, we also need to collect observation data set $u(x_i, t_i)$ where $(x_i, t_i)\in\D$, and define the data loss as
\begin{equation}\label{eq:loss_data}
\textstyle \Ldata:= \frac{1}{|\D|}
  \sum_{(t_i,x_i)\in\D}
  \left( u_\theta(t_i,x_i) - u_i \right)^2,
\end{equation}
and similar to solving the forward problem classical approach solves this by solving a minimization problem on a weighted loss that is a weighted summation of all the losses and then use first order method to find a solution.
\begin{align}\label{eq:inverse}
  \textstyle  \min_{\theta,\kappa} \cL(\theta,\kappa):=\wphy\Lphy(\theta,\kappa) +\wic \Lic(\theta) + \wbc\Lbc(\theta) + \wdata\Ldata(\theta)
\end{align}


\section{Feedback Linearization Constrained Optimization for Solving PINNs}

\paragraph{Constrained Optimization for PINNs} As mentioned in the previous section, PINNs typically solve optimization problems where the total objective loss $\cL(\theta), \cL(\theta,\kappa)$ is a linear combination of different losses such as data loss, physics loss, boundary condition loss and initial condition loss. While these constraints ideally should be satisfied during training, the highly nonlinear
structure of the network and the discretized PDE makes it difficult to achieve objective loss going to strictly zero. Further, it often requires delicate tuning in the penalty weights $\wdata,\wphy,\wic,\wbc$.

Another way to handle this is via constrained optimization rather than optimizing the linear combination of different loss. Instead of solving \eqref{eq:forward} or \eqref{eq:inverse}, rather, we choose to put some of the losses into the constraint. In specific, for the \emph{forward problem}, we consider the objective function is the physics loss $\Lphy(\theta)$, and the equality constraints aggregate the initial and boundary condition losses into a single vector $h(\theta) := [\Lic(\theta), \Lbc(\theta)]^\top \in \mathbb{R}^2$:
\begin{equation}\label{eq:forward-constrained}
\begin{split}
    \textstyle \min_\theta~~ \Lphy(\theta)
    \qquad s.t. ~~\Lbc(\theta)=0,~~ \Lic(\theta)=0
\end{split}
\end{equation}
and for the inverse problem we consider the optimization variable expands to include both the neural network weights and the unknown PDE parameters, effectively replacing $\theta$ with the joint variable $(\theta, \kappa)$. The objective becomes the data loss $f(\theta, \kappa) := \Ldata(\theta)$, and the constraint vector expands to enforce the PDE residual alongside the initial and boundary conditions, yielding $[\Lphy(\theta, \kappa), \Lic(\theta), \Lbc(\theta)]^\top \in \mathbb{R}^3$. 
\begin{equation}\label{eq:inverse-constrained}
\begin{split}
   \textstyle  \min_{\theta,\kappa}~~ \Ldata(\theta)\qquad
    s.t.~~   \Lphy(\theta,\kappa), ~\Lbc(\theta),~\Lic(\theta) = 0
\end{split}
\end{equation}

Hence, rather than consider a unconstrained optimization problem, we put some of the losses as constraint and then we can use constrained optimization solvers. Note that both the forward problem \ref{eq:forward-constrained} and inverse problem \ref{eq:inverse-constrained} fit into the general equality-constrained optimization framework as follows
\begin{align}\label{eq:eq-constrained-opt}
 \textstyle   \min_x f(\theta) \quad \text{s.t.} \quad h(\theta)=0.
\end{align}

\begin{remark}[Why Constrained Optimization for PINNs\label{remark-1} ]
Comparing with the penalty-based formulations \eqref{eq:forward}, \eqref{eq:inverse}, constrained optimization offers a principled alternative by explicitly treating physical laws and boundary conditions as constraints, rather than soft penalties, thereby providing a clearer separation between feasibility and optimality. Moreover, constrained formulations can be interpreted as implicitly performing adaptive penalty adjustment (e.g. \cite{li2022dynamic}) through dual variables, alleviating the need for manual loss reweighting and improving robustness across different problem instances.

We observe in Table~\ref{tab:combined_results} that constrained methods achieve substantially smaller boundary and initial condition errors than unconstrained optimizers such as Adam, even though the physics loss is occasionally larger, which is consistent with prior observations~\cite{rathore2024challenges,wang2024understanding,de2024error}. This behavior can be attributed to two factors: first, the physics loss computed via automatic differentiation is sensitive to local approximation errors and network smoothness; second, PINN training involves an inherent tradeoff among physics residuals, boundary conditions, and initial conditions, so enforcing feasibility more strictly can redistribute errors across these components. 

\end{remark}
\paragraph{Feedback Linearization (FL) for equality-constrained optimization.}
Having reformulated PINN training as an equality-constrained problem, the next question is how to solve it efficiently and robustly. In this work, inspired by recent control-theoretic approaches to constrained optimization~\cite{zhang2025constrained,cerone2025new}, we adopt a feedback linearization (FL) perspective, which views optimization dynamics as a controlled system and designs the multiplier to directly shape constraint evolution.

Consider the equality-constrained problem in \eqref{eq:eq-constrained-opt}, whose first-order KKT conditions are
\begin{align*}
\nabla f(\theta) + J_h(\theta)^\top \lambda = 0, \qquad h(\theta)=0.
\end{align*}
A common approach to constrained optimization is to perform gradient descent on the Lagrangian,
\begin{align*}
\theta_{t+1} = \theta_t - \eta\big(\nabla f(\theta_t) + J_h(\theta_t)^\top \lambda_t\big),
\end{align*}
which reduces the problem to designing the dual variable \(\lambda_t\).

Following~\cite{zhang2025constrained,cerone2025new}, we interpret \(\theta\) as the system state and \(\lambda\) as a feedback input, leading to a control-theoretic viewpoint in which the goal is to design a stabilizing feedback law. Let $y_t = h(\theta_t)$ denote the constraint violation. A first-order approximation gives
\begin{align*}
y_{t+1} - y_t \approx -\eta\left(J_h(\theta_t)\nabla f(\theta_t) + J_h(\theta_t)J_h(\theta_t)^\top \lambda_t\right).
\end{align*}
FL selects \(\lambda_t\) to impose stable linear dynamics on \(y_t\), namely
\begin{align*}
y_{t+1} - y_t = -\eta K y_t,
\end{align*}
where \(0 \prec K \preceq \tfrac{1}{\eta} I\). Solving for \(\lambda_t\) yields
\begin{align}\label{eq:lambda}
\lambda_t = -\big(J_h(\theta_t)J_h(\theta_t)^\top\big)^{-1}
\big(J_h(\theta_t)\nabla f(\theta_t) - K h(\theta_t)\big),
\end{align}
which enforces a controlled contraction of constraint violations. Substituting back gives the FL update
\vspace{-10pt}
\begin{align*}
\quad \theta_{t+1} \!=\! \theta_t \!-\! \eta\left(\nabla f(\theta_t) - J_h(\theta_t)^\top \big(J_h(\theta_t)J_h(\theta_t)^\top\big)^{-1}
\big(J_h(\theta_t)\nabla f(\theta_t) \!-\! K h(\theta_t)\big)\right).
\end{align*}
The analysis and empirical results in~\cite{cerone2025new, zhang2025constrained} show that the resulting algorithm performs strongly and achieves $O(\frac{1}{\sqrt{T}})$ even in nonconvex settings. 
\begin{remark}[Connection and Difference with SQP]\label{rmk:connection-with-SQP}
Recent results in~\cite{zhang2025constrained} establish a formal connection between FL and sequential quadratic programming (SQP), showing that for the specific choice $K=\frac{1}{\eta} I$, the FL update recovers the following SQP
\begin{align*}
\textstyle \min_{d} ~~  \nabla f(\theta_t)^\top d + \tfrac{1}{2\eta}\|d\|^2 \qquad
\text{s.t.} ~~  h(\theta_t) + J_h(\theta_t)d = 0.
\end{align*}
A key distinction, however, is that FL introduces an explicit gain matrix K to directly shape the constraint dynamics, enabling flexible control over the stability and rate of feasibility enforcement. Moreover, the control-theoretic formulation of FL naturally admits momentum extensions that incorporate temporal smoothing into the dynamics, and this observation directly motivates the momentum-based variants developed next.
\end{remark}
\paragraph{FL with Momentum (FL-momentum)}
In \cite{zhang2025constrained}, a momentum-based variant of feedback linearization, termed FL-momentum, was proposed to accelerate convergence. The method augments the feedback-linearized update with a momentum variable as
\begin{align*}
    m_{t} &= \beta m_{t-1} + \left(\nabla f(\theta_t) + J_h(\theta_t)^\top \lambda_t \right), 
    \qquad \left(\lambda_t \text{ is set according to \eqref{eq:lambda}}\right) \\
    \theta_{t+1} &= \theta_t - \eta m_t .
\end{align*}
Compared to naive FL, which updates the primal variable $\theta$ using only the instantaneous feedback-linearized gradient
$\nabla f(\theta_t) + J_h(\theta_t)^\top \lambda_t$, FL-momentum incorporates information from previous iterations through the momentum variable $m_t$. From an optimization perspective, the momentum term introduces an inertial effect that extrapolates the update direction based on past descent trends, allowing the iterate to move more aggressively along persistent descent directions. As a result, FL-momentum empirically achieves accelerated convergence relative to naive FL.

\section{AdamFLIP: An Optimizer for Solving Constrained Optimization in PINNs}
While FL-momentum introduced at the end of previous section achieves acceleration, it still uses a single global stepsize across all coordinates, which can be suboptimal for neural-network parametrizations where gradients are often highly anisotropic and unevenly scaled. In contrast, in unconstrained optimization, Adam~\cite{kingma2014adam} has achieved strong empirical success in training neural networks across a wide range of applications. Motivated by this observation, we propose AdamFLIP (\textbf{Ada}ptive \textbf{m}omentum \textbf{F}eedback \textbf{Li}nearization Optimization for Hard Constrained \textbf{P}INN Training), which augments FL-momentum with coordinate-wise adaptive scaling; see Algorithm~\ref{alg:AdamFLIP}.
\begin{algorithm}[htbp]
\caption{AdamFLIP Algorithm}
\label{alg:AdamFLIP}
\begin{algorithmic}[1]
\Require Initial parameters $\theta_0$, step size $\eta$, moments $m_0=0$, $v_0=0$, hyperparameters $\beta_1, \beta_2$, $\epsilon, \delta$
\For{$t = 1,2,\dots$}
    \State Compute constraint multiplier $\lambda_t =  -\big(J_h(\theta_t)J_h(\theta_t)^\top \!\!+ \epsilon I\big)^{-1}
    \big(J_h(\theta_t)\nabla f(\theta_t) \!-\! K h(\theta_t)\big)$ 
    \State $g_t \gets \nabla f(\theta_t) + J_h(\theta_t)^\top \lambda_t$
    \State $m_t \gets \beta_1 m_{t-1} + (1-\beta_1) g_t$
    \State $v_t \gets \beta_2 v_{t-1} + (1-\beta_2) g_t^{\circ 2}$
    \State $\hat{m}_t \gets \dfrac{m_t}{1-\beta_1^t},\quad \hat{v}_t \gets \dfrac{v_t}{1-\beta_2^t}$
    \State $\theta_{t+1} \gets \theta_t - \eta \, \hat{m}_t \oslash \left(\hat{v}_t^{\circ \frac{1}{2}} + \delta \right)$
\EndFor
\end{algorithmic}
\end{algorithm}
where $x^{\circ 2}$ and $x^{\circ \frac{1}{2}}$ denote the elementwise square and square root of $x$, and $x \oslash y$ denotes elementwise division. 

In addition to the first-moment estimate $m_t$, AdamFLIP maintains a second-moment accumulator $v_t$ that tracks the elementwise squared magnitude of the feedback-linearized gradient. The resulting update rescales each coordinate of the momentum direction by an estimate of its historical variability, allowing parameters with consistently large gradients to take smaller steps while amplifying progress along flatter directions.

The design of AdamFLIP is directly inspired by the Adam optimizer. Specifically, AdamFLIP mirrors Adam's use of bias-corrected first- and second-moment estimates, but replaces the standard gradient with the gradient of the Lagrangian $\nabla f(\theta_t) + J_h(\theta_t)^\top \lambda_t$. From this perspective, AdamFLIP can be interpreted as applying Adam-style adaptive normalization to feedback-linearized constrained optimization, preserving the structure and intuition of FL while inheriting the robustness and scalability that make Adam effective in large-scale neural network training.

We would also like to point out one caveat that FL-based algorithms are designed under the assumption that $J_h(\theta)J_h(\theta)^\top$ is invertible, an assumption that may not hold in PINN training, particularly near exact constraint satisfaction where some residual gradients can become near zero. In Algorithm~\ref{alg:AdamFLIP}, we mitigate potential ill-conditioning by introducing a damped solve using $J_h(\theta_t)J_h(\theta_t)^\top+\epsilon I$.


\subsection{Theoretical Guarantees}
\label{sec:theory}
In this section, we provide a finite-time convergence guarantee for a variant of AdamFLIP. We note that a rigorous analysis of the original AdamFLIP remains challenging. To address this, we introduce a modified version (Algorithm~\ref{alg:metric-adamflip}), for which we establish convergence.

The key difficulty with the original AdamFLIP (Algorithm~\ref{alg:AdamFLIP}) lies in a mismatch of metrics: the feedback-linearization multiplier is computed in the Euclidean metric, while the final update is preconditioned by an Adam-style diagonal metric $D_t = \diag\left(\frac{1}{\sqrt{\bar v_{t}}+\delta}\right)$. This inconsistency makes it difficult to control the optimization dynamics. To resolve this issue, the variant computes the multiplier directly in the same diagonal metric $D_t$ used for the update, ensuring metric consistency throughout. In addition, following \cite{reddi2018convergence}, we adopt a monotone second-moment estimate (AMSGrad-style) to stabilize the adaptive preconditioner. This modification prevents oscillations in the effective stepsizes and enables a tractable convergence analysis.

We emphasize that, despite strong empirical performance, providing a convergence guarantee for the original AdamFLIP remains an open problem. We leave this question to future work. In the following, we introduce standard assumptions for nonconvex constrained optimization (cf.~\cite{nocedal2006numerical, peterson1973review}).




\begin{algorithm}[htbp]
\caption{AdamFLIP variant }
\label{alg:metric-adamflip}
\begin{algorithmic}[1]
\Require Initial point $\theta_1$, moments $m_0=0$, $v_0=0$, $\bar v_0=0$, stepsizes $\{\eta_t\}$, parameters $\beta_1,\beta_2\in[0,1)$, denominator constant $\delta>0$, gain $K=\kappa I$
\For{$t=1,2,\dots$}
    \State Define $D_t=\diag\left(\frac{1}{\sqrt{\bar v_{t-1}}+\delta}\right).$
    \State Compute
       $ \lambda_t=-\big(J_h(\theta_t)D_tJ_h(\theta_t)^\top\big)^{-1}
        \big(J_h(\theta_t)D_t\nabla f(\theta_t)-\kappa h(\theta_t)\big).$
    \State $g_t=\nabla f(\theta_t)+J_h(\theta_t)^\top\lambda_t.$
    \State $m_t=\beta_1m_{t-1}+(1-\beta_1)g_t,\quad v_t=\beta_2v_{t-1}+(1-\beta_2)g_t^{\circ2}.$
    \State $\hat m_t=\frac{m_t}{1-\beta_1^t},\quad \tilde v_t=\frac{v_t}{1-\beta_2^t}.$
    \State $\bar v_t=\max_{\rm elem}\{\bar v_{t-1},\tilde v_t\}.$
    \State $\theta_{t+1}=\theta_t-\eta_tD_t\hat m_t.$
\EndFor
\end{algorithmic}
\end{algorithm}



\begin{assumption}[Smoothness and regularity]
\label{assump:smooth-compact}
$f,h$ are $C^2$ on a neighborhood of a compact set $\mathcal C$, the iterates remain in $\mathcal C$, and $J_h(\theta)$ has full row rank on $\mathcal C$. Consequently, there exists $G$ such that $\|g_t\|\le G$ for all $t$.
\end{assumption}

\begin{assumption}[Stepsizes]
\label{assump:stepsize}
$\eta_t>0$, $\eta_{t+1}\le\eta_t$, $\eta_t\le 1/\kappa$, $\sum_t\eta_t=\infty$, $\sum_t\eta_t^2<\infty$.
\end{assumption}

Define
    $M_t:=J_h(\theta_t)D_tJ_h(\theta_t)^\top,\quad
    \lambda_t^\dagger:=-M_t^{-1}J_h(\theta_t)D_t\nabla f(\theta_t),$
and
    $r_t:=\nabla f(\theta_t)+J_h(\theta_t)^\top\lambda_t^\dagger,$
then we can prove the following theorem

\begin{theorem}[Finite-time convergence, proof see Appendix \ref{sec:proof-main}]
\label{thm:main}
Under Assumptions~\ref{assump:smooth-compact}--\ref{assump:stepsize}, the iterates generated by Algorithm~\ref{alg:metric-adamflip} satisfy
\begin{align*}
\textstyle    \sum_{t=1}^T \eta_t \left(\|r_t\|^2 + \|h(\theta_t)\|_1 \right)
    = \mathcal{O}\left(1 + \sum_{t=1}^T \eta_t^2\right).
\end{align*}
Consequently,
\vspace{-15pt}
\begin{align*}
    \textstyle \min_{1\le t\le T} \left(\|r_t\|^2 + \|h(\theta_t)\|_1 \right)
    = \mathcal{O}\left(\frac{1}{\sum_{t=1}^T \eta_t}\right).
\end{align*}
\end{theorem}
\begin{remark}[Discussion on the convergence implications and assumptions]
The quantity $\|r_t\|$ corresponds to a first-order stationarity residual, while $\|h(\theta_t)\|_1$ measures feasibility violation. Together, $\|r_t\|^2 + \|h(\theta_t)\|_1$ serves as a natural surrogate for the KKT residual of the constrained problem. Consequently, Theorem~\ref{thm:main} establishes that this KKT residual diminishes along the iterates in a weighted average sense, which in turn implies a best-iterate convergence guarantee.

In particular, choosing $\eta_t = \frac{\eta}{\sqrt{t}\log t}$ for $t\ge 2$ satisfies the stepsize conditions. In this case, Theorem~\ref{thm:main} implies the best-iterate guarantee
$$
\textstyle \min_{2\le t\le T} \left(\|r_t\|^2 + \|h(\theta_t)\|_1 \right)
= \mathcal{O}\!\left(\frac{\log T}{\sqrt{T}}\right).
$$
 We note, however, that the obtained rate is still weaker than the classical rate typically achieved by first-order methods for nonconvex constrained optimization with $\textstyle \min_{1\le t\le T} \left(\|r_t\|^2 + \|h(\theta_t)\|_1 \right)
= \mathcal{O}\!\left(\frac{1}{{T}}\right)$. This gap reflects the additional challenges introduced by the adaptive metric and the coupled multiplier update, and closing this gap remains an interesting direction for future work.

We also note that the full-row-rank condition on $J_h(\theta)$ cannot be guaranteed in general for PINN training, particularly near exact constraint satisfaction where some residual gradients may become degenerate. In practice, we do not observe this issue to be problematic, and we further stabilize the multiplier computation by using a damped solve with $J_h(\theta_t)J_h(\theta_t)^\top+\epsilon I$. Extending the convergence analysis to rank-deficient or relaxed constraint settings remains an open direction.
\end{remark}
\section{Experiment Results}\label{sec:results}
\vspace{-10pt}

\begin{table}[htbp]
\centering
\caption{Consolidated error comparison of evaluated models. Metrics include relative $L^2$ error, wall-clock time (WT, in seconds) measured on an NVIDIA H100 GPU (80\,GB), and constraint errors (CE) as MSE across initial/boundary conditions (IC/BC), physics residuals (Phy), and data loss (DL). Parameter recovery absolute errors are denoted by Err. Best results are highlighted in {\colorbox{bestblue!60}{blue}} and second-best in \colorbox{bestblue!25}{light blue}.}
\label{tab:combined_results}
\setlength{\tabcolsep}{4pt}
\adjustbox{max width=\textwidth}{%
\begin{tabular}{llccccc}
\toprule
Dataset & Metric & Standard PINN & AdamFLIP & FL-PINN & AL-PINN & trSQP-PINN \\
\midrule
 & Rel. $L_2$ & $2.51\times10^{-1}$ & \first{$3.12\times10^{-2}$} & $1.06\times10^{-1}$ & \second{$8.20\times10^{-2}$} & $4.43\times10^{-1}$ \\
 & CE (IC) & $4.26\times10^{-3}$ & \first{$1.00\times10^{-4}$} & \second{$6.82\times10^{-4}$} & $2.34\times10^{-3}$ & $9.91\times10^{-2}$ \\
Burgers Fwd & CE (BC) & $2.44\times10^{-3}$ & \second{$1.00\times10^{-4}$} & $4.40\times10^{-4}$ & \first{$6.25\times10^{-5}$} & $9.80\times10^{-2}$ \\
 & CE (Phy) & \second{$1.49\times10^{-3}$} & $2.60\times10^{-3}$ & $1.13\times10^{-2}$ & $2.97\times10^{-3}$ & \first{$1.31\times10^{-3}$} \\
 & WT (s) & \first{113} & \second{115} & 133 & 155 & 625 \\
\midrule

 & Rel. $L_2$ & $5.46\times10^{-1}$ & \first{$5.74\times10^{-2}$} & $4.19\times10^{-1}$ & \second{$4.02\times10^{-1}$} & $4.54\times10^{-1}$ \\
 & CE (IC) & \second{$9.00\times10^{-3}$} & \first{$2.90\times10^{-3}$} & $2.26\times10^{-2}$ & $3.06\times10^{-2}$ & $4.41\times10^{-1}$ \\
Burgers Inv & CE (BC) & \first{$1.00\times10^{-4}$} & \second{$3.50\times10^{-3}$} & $2.04\times10^{-2}$ & $2.73\times10^{-2}$ & $2.56\times10^{-1}$ \\
 & CE (Phy) & \first{$1.00\times10^{-4}$} & $2.11\times10^{-2}$ & \second{$1.80\times10^{-3}$} & $2.48\times10^{-2}$ & $5.20\times10^{-3}$ \\
 & CE (DL) & $1.13\times10^{-1}$ & \first{$1.20\times10^{-3}$} & $6.62\times10^{-2}$ & \second{$5.94\times10^{-2}$} & $4.07\times10^{-1}$ \\
 & Err ($\kappa_1$) & \second{$3.33\times10^{-2}$} & \first{$1.54\times10^{-2}$} & $1.28\times10^{-1}$ & $5.94\times10^{-1}$ & $9.82\times10^{-1}$ \\
 & Err ($\kappa_2$) & \second{$6.50\times10^{-3}$} & \first{$3.60\times10^{-3}$} & $9.65\times10^{-2}$ & $7.97\times10^{-2}$ & $1.90\times10^{-2}$ \\
 & WT (s) & \first{248} & \second{324} & 334 & 334 & 621 \\
\midrule

 & Rel. $L_2$ & $7.71\times10^{-1}$ & \first{$2.85\times10^{-1}$} & \second{$4.26\times10^{-1}$} & $4.93\times10^{-1}$ & $4.56\times10^{-1}$ \\
 & CE (IC) & $1.65\times10^{-3}$ & \first{$1.42\times10^{-5}$} & \second{$1.36\times10^{-4}$} & $1.11\times10^{-3}$ & $1.69\times10^{-3}$ \\
TFMDWEs Fwd & CE (BC) & $2.19\times10^{-2}$ & \first{$1.47\times10^{-4}$} & \second{$4.20\times10^{-3}$} & $5.16\times10^{-2}$ & $2.90\times10^{-2}$ \\
 & CE (Phy) & \second{$1.07\times10^{-1}$} & \first{$8.87\times10^{-2}$} & $1.19\times10^{-1}$ & $7.88\times10^{-1}$ & $8.13\times10^{-1}$ \\
 & WT (s) & 56 & \first{50} & \second{52} & 53 & 105 \\
\midrule

 & Rel. $L_2$ & \second{$2.22\times10^{-1}$} & \first{$3.95\times10^{-2}$} & $5.99\times10^{-1}$ & $8.81\times10^{-1}$ & $6.67\times10^{-1}$ \\
 & CE (IC) & $9.01\times10^{-4}$ & \second{$1.77\times10^{-5}$} & \first{$1.56\times10^{-5}$} & $2.60\times10^{-4}$ & $1.19\times10^{-4}$ \\
TFMDWEs Inv & CE (BC) & $2.27\times10^{-2}$ & \first{$2.57\times10^{-4}$} & \second{$4.08\times10^{-4}$} & $3.74\times10^{-3}$ & $3.59\times10^{-3}$ \\
 & CE (Phy) & $7.01\times10^{-2}$ & $6.53\times10^{-2}$ & $1.82\times10^{-1}$ & \second{$4.92\times10^{-3}$} & \first{$3.78\times10^{-3}$} \\
 & CE (DL) & \second{$3.59\times10^{-3}$} & \first{$1.13\times10^{-4}$} & $2.61\times10^{-2}$ & $5.62\times10^{-2}$ & $5.44\times10^{-2}$ \\
 & Err ($\alpha$) & $1.93\times10^{-1}$ & \first{$3.40\times10^{-2}$} & \second{$7.40\times10^{-2}$} & $9.54\times10^{-2}$ & $1.64\times10^{-1}$ \\
 & WT (s) & \first{120} & \second{133} & 141 & 134 & 144 \\
\midrule

 & Rel. $L_2$ & $3.43\times10^{-2}$ & \first{$7.53\times10^{-3}$} & \second{$9.49\times10^{-3}$} & $3.09\times10^{-2}$ & $6.60\times10^{-1}$ \\
 & CE (IC) & $1.20\times10^{-4}$ & \second{$5.59\times10^{-5}$} & $1.13\times10^{-4}$ & \first{$3.65\times10^{-5}$} & $1.22\times10^{-1}$ \\
Heat Fwd & CE (BC) & $1.56\times10^{-4}$ & \first{$6.35\times10^{-5}$} & $1.70\times10^{-4}$ & \second{$7.89\times10^{-5}$} & $2.62\times10^{-2}$ \\
 & CE (Phy) & $1.82\times10^{-3}$ & \first{$2.00\times10^{-4}$} & \second{$7.35\times10^{-4}$} & $7.41\times10^{-4}$ & $1.07\times10^{-2}$ \\
 & WT (s) & \first{261} & \second{336} & 354 & 364 & 1260 \\
\midrule

 & Rel. $L_2$ & \second{$2.90\times10^{-2}$} & \first{$1.59\times10^{-2}$} & $4.06\times10^{-2}$ & $3.98\times10^{-2}$ & $6.42\times10^{-1}$ \\
 & CE (IC) & $3.06\times10^{-5}$ & \first{$1.43\times10^{-5}$} & \second{$1.80\times10^{-5}$} & $1.71\times10^{-4}$ & $4.81\times10^{-4}$ \\
Heat Inv & CE (BC) & $1.19\times10^{-4}$ & \second{$1.03\times10^{-5}$} & \first{$8.73\times10^{-6}$} & $3.07\times10^{-4}$ & $4.57\times10^{-3}$ \\
 & CE (Phy) & $2.79\times10^{-4}$ & \second{$9.13\times10^{-5}$} & \first{$7.44\times10^{-5}$} & $2.29\times10^{-4}$ & $3.73\times10^{-3}$ \\
 & CE (DL) & \second{$5.22\times10^{-5}$} & \first{$1.51\times10^{-5}$} & $8.81\times10^{-5}$ & $9.18\times10^{-5}$ & $2.59\times10^{-2}$ \\
 & Err ($\kappa$) & \second{$9.00\times10^{-4}$} & \first{$5.00\times10^{-4}$} & $6.70\times10^{-3}$ & $2.00\times10^{-3}$ & $8.84\times10^{-1}$ \\
 & WT (s) & \first{297} & \second{346} & 360 & 361 & 1300 \\
\midrule

 & Rel. $L_2(u)$ & \second{$3.62\times10^{-2}$} & \first{$1.18\times10^{-2}$} & $7.70\times10^{-2}$ & $5.19\times10^{-2}$ & $1.96\times10^{0}$ \\
 & Rel. $L_2(v)$ & \second{$3.48\times10^{-2}$} & \first{$1.06\times10^{-2}$} & $8.13\times10^{-2}$ & $3.69\times10^{-2}$ & $9.68\times10^{-1}$ \\
Navier--Stokes & CE (IC) & \second{$2.15\times10^{-5}$} & \first{$1.26\times10^{-5}$} & $1.98\times10^{-3}$ & $1.10\times10^{-3}$ & $7.82\times10^{-4}$ \\
 & CE (BC) & \second{$1.02\times10^{-4}$} & \first{$4.86\times10^{-5}$} & $2.72\times10^{-3}$ & $8.13\times10^{-4}$ & $3.56\times10^{-4}$ \\
 & CE (Phy) & \second{$6.11\times10^{-4}$} & \first{$2.53\times10^{-4}$} & $1.08\times10^{-2}$ & $8.56\times10^{-4}$ & $1.29\times10^{1}$ \\
 & WT (s) & \first{155} & \second{335} & 414 & 410 & 720 \\

\bottomrule
\end{tabular}%
}
\end{table}
We evaluate AdamFLIP on four benchmark PDEs spanning 1D and 2D settings, linear and nonlinear dynamics, and integer- and fractional-order operators: (i) 1D viscous Burgers equation, (ii) 1D time-fractional mixed diffusion-wave equation (TFMDWE), (iii) 2D heat equation, and (iv) 2D incompressible Navier--Stokes equations (NSE). Each problem is tested in both \emph{forward} (all PDE parameters known) and \emph{inverse} (unknown parameters co-optimized with the network) settings, except NSE which is tested in the forward setting only. We compare against a standard PINN (Adam), AL-PINN~\cite{son2023enhanced}, trSQP-PINN~\cite{cheng2024physics}, and FL-PINN (feedback linearization without adaptive moments), as summarized in Table~\ref{tab:model_acronyms}. The rationale for selecting these baselines is discussed in Appendix~\ref{sec:related-works} and Remark~\ref{rmk:baseline-choice}. All methods share identical network architectures, training data, and stopping criteria per problem to ensure a fair comparison. Full experimental details (architecture, collocation points, hyperparameters, et al.) are deferred to Appendix~\ref{app:detailed_results}.

\paragraph{Burgers' equation.}
We consider the 1D viscous Burgers equation on $\Omega = [-1,1]$, $t \in [0,T]$:
\begin{align}
&u_t + u u_x- \nu u_{xx}  = 0,\quad u(\pm 1, t) = 0,\quad u(x,0) = -\sin(\pi x),\label{eq:burgers}
\end{align}
with $\nu = 0.01/\pi$. The nonlinear advection--diffusion coupling produces steep shock formation, making this a standard yet challenging PINN benchmark. In the inverse setting, two unknown parameters $\kappa_1$ (scaling $u u_x$) and $\kappa_2$ (scaling $u_{xx}$) are co-optimized with the network, initialized at $\kappa_1{=}2,\,\kappa_2{=}0$ (true values: $\kappa_1{=}1$, $\kappa_2{=}\nu$).

\paragraph{Time-fractional mixed diffusion-wave equation (TFMDWE).}
On $\Omega = (0,\pi)$, $t \in [0,1]$:
\begin{align}
&D_t^{\alpha} u - u_{xx} = f(x,t), \quad u\big|_{\partial\Omega} = 0, \quad u(x,0) = 0, \label{eq:tfmdwe}
\end{align}
where $D_t^{\alpha}$ is the Caputo fractional derivative of order $\alpha \in (0,1]$, modeling anomalous diffusion in viscoelastic media~\cite{luchko2011fractional}. With $f(x,t) = [\Gamma(4)/\Gamma(4{-}\alpha)\, t^{3-\alpha} + t^3]\sin(x)$, the exact solution is $u = t^3\sin(x)$. In the inverse setting, $\alpha$ is treated as an unknown parameter.
\paragraph{2D heat equation.}
On $\Omega = (0,1)^2$, $t \in [0,1]$:
\begin{align}
&u_t - \nu \Delta u = 0, \quad u\big|_{\partial\Omega} = 0, \quad u(x,y,0) = \sin(\pi x)\sin(\pi y), \label{eq:2d-heat}
\end{align}
with analytical solution $u = \sin(\pi x)\sin(\pi y)\, \mathrm{e}^{-2\pi^2 \nu t}$. In the inverse setting, the thermal diffusivity $\nu$ (denoted $\kappa$) is unknown and co-optimized.

\paragraph{2D Navier--Stokes equations.}
On $\Omega = [0,2\pi]^2$, $t \in [0,T]$:
\begin{align}
&\mathbf{u}_t + (\mathbf{u} \cdot \nabla)\mathbf{u} = -\nabla p + \nu \Delta \mathbf{u}, \quad \nabla \cdot \mathbf{u} = 0, \label{eq:ns}
\end{align}
where $\mathbf{u} = (u,v)$ is the velocity, $p$ is the pressure, and $\nu = 0.01$. We benchmark on the Taylor--Green vortex (closed-form solution). The network maps $(\mathbf{x}, t) \mapsto (\mathbf{u}, p)$.

\paragraph{Discussion.}
Table~\ref{tab:combined_results} shows all results. For all seven problem configurations, AdamFLIP achieves the lowest relative $L_2$ error, , typically improving over the next-best baseline by a factor of $2$ to $10$. We highlight several key observations.

\textit{Forward problems.} On Burgers (forward), AdamFLIP attains a relative $L_2$ error of $3.12{\times}10^{-2}$, roughly $3{\times}$ lower than the next-best AL-PINN ($8.20{\times}10^{-2}$) and an order of magnitude lower than Standard PINN ($2.51{\times}10^{-1}$), while simultaneously achieving the best IC/BC constraint satisfaction. On the 2D heat equation, AdamFLIP ($7.53{\times}10^{-3}$) and FL-PINN ($9.49{\times}10^{-3}$) both substantially outperform Standard PINN ($3.43{\times}10^{-2}$), confirming the benefit of constrained formulations; notably, trSQP-PINN diverges ($6.60{\times}10^{-1}$) on this problem. For the TFMDWE, all methods have difficulty due to the non-local fractional operator, yet AdamFLIP ($2.85{\times}10^{-1}$) still improves upon the next-best FL-PINN ($4.26{\times}10^{-1}$) while reducing IC/BC errors by one--two orders of magnitude. On the Navier--Stokes benchmark, AdamFLIP reduces $L_2$ errors for both velocity fields to ${\sim}1{\times}10^{-2}$, roughly $3{\times}$ better than Standard PINN and an order of magnitude better than FL-PINN; trSQP-PINN again fails catastrophically on this system.

\textit{Inverse problems.} The advantages of AdamFLIP are even more pronounced in parameter inferring. On Burgers (inverse), AdamFLIP achieves $L_2 = 5.74{\times}10^{-2}$, nearly an order of magnitude below all baselines (${\sim}4{-}5{\times}10^{-1}$), and recovers $\kappa_1,\kappa_2$ with the smallest absolute errors ($1.54{\times}10^{-2}$ and $3.60{\times}10^{-3}$, respectively). For the TFMDWE inverse problem, AdamFLIP attains $L_2 = 3.95{\times}10^{-2}$ and the most accurate recovery of the fractional order $\alpha$ (error $3.40{\times}10^{-2}$), while AL-PINN and FL-PINN degrade to errors larger than $5{\times}10^{-1}$. On the 2D heat inverse problem, all methods except trSQP-PINN perform reasonably, with AdamFLIP ($1.59{\times}10^{-2}$) achieving the best $L_2$ accuracy and parameter error ($5.00{\times}10^{-4}$).

\textit{Computational cost.} AdamFLIP incurs only a modest overhead relative to Standard PINN (typically ${\leq}30\%$), while trSQP-PINN requires $2{-}5{\times}$ longer wall-clock time due to trust-region subproblem solves and pretraining requirements. AL-PINN and FL-PINN have comparable cost to AdamFLIP.

\textit{Constraint satisfaction trade-offs.} Standard PINN occasionally achieves the lowest physics-loss CE (e.g., Burgers inverse), yet this does not translate to more accurate solution, consistent with the observation in Remark \ref{remark-1} that lower physics residual does not imply better constraint satisfaction in solution space. AdamFLIP's feedback mechanism adaptively redistributes enforcement effort across IC, BC, and physics constraints, yielding the best overall accuracy.

Additional spatiotemporal solution fields, absolute error maps, and per-problem analysis are provided in Appendix~\ref{app:detailed_results}.

\section{Conclusion and Future Work}

In this paper, we introduced AdamFLIP, an adaptive, momentum-based feedback linearization optimizer designed for hard-constrained PINN training. By formulating a feedback mechanism that drives constraint violations toward a stable linear dynamical system, and integrating Adam-style momentum updates, AdamFLIP enforces hard constraints and fasten the convergence. Crucially, it aviods the process of manual loss-weight tuning while maintaining computational efficiency. We evaluated our method against state-of-the-art constrained PINN baselines on classic PDEs and time-fractional partial differential equations (TFMDWEs) across both forward and inverse problem settings. Empirical results demonstrate that AdamFLIP not only guarantees rigorous constraint satisfaction but also yields lower relative $L_2$ errors with only a marginal increase in training time. 

These findings open several promising avenues for future research. First, extending this first-order framework to zeroth-order optimization~\cite{zhang2025zeroth} could further enhance both model expressiveness and interpretability. Second, applying AdamFLIP to naturally enforce PDE symmetries, specifically, Lie point symmetries \cite{brandstetter2022lie, akhound2023lie, mialon2023self}, presents a compelling direction for developing even more physically rigorous neural solvers.
\begin{ack}
Guang Lin gratefully acknowledges support from the National Science Foundation
(DMS-2533878, DMS-2053746, DMS-2134209, ECCS-2328241, CBET-2347401, and OAC-2311848),
the U.S.~Department of Energy (DOE) Office of Science Advanced Scientific Computing
Research program under award DE-SC0023161, the SciDAC LEADS Institute, and the
DOE Fusion Energy Sciences program under award DE-SC0024583. Runyu Zhang is
supported by the MIT Postdoctoral Fellowship Program for Engineering Excellence.
Na Li gratefully acknowledges support in part from NSF ASCENT under Grant 2328241
and in part from the NSF AI Institute under Grant 2112085.
\end{ack}

\bibliographystyle{unsrt}  
\bibliography{references} 

\appendix

\section{Related Works: Constrained Optimization Methods for PINN, and comparison with AdamFLIP}\label{sec:related-works}

In this section, we review recent constrained-optimization approaches for PINN training and compare their algorithmic principles with AdamFLIP.

One line of work improves the optimization dynamics while largely preserving the penalty-based PINN objective. For example, \cite{xu2021trust} uses a trust-region strategy that restricts each update to a region where a local model of the composite loss is reliable, thereby improving stability in ill-conditioned regimes. ADMM-PINNs~\cite{song24admmpinn} instead use operator splitting: auxiliary variables are introduced so that the original problem decomposes into a PINN-based smooth subproblem and additional subproblems that may admit proximal or closed-form updates. Although these methods introduce constraints, such as the trust-region constraint in trust-region methods and the variable-splitting consistency constraint in ADMM, the main PINN training objective still relies on a weighted combination of PDE, data, and boundary/initial residuals. Therefore, they remain closer to soft-penalty-based PINN methods rather than fully hard-constrained formulations.

A more fundamental shift is to recast PINN training itself as a constrained optimization problem. In this view, some loss components are no longer treated merely as weighted penalties but are instead imposed as equality constraints. A prominent family of methods follows the augmented Lagrangian (AL) principle~\cite{lu2021physics_hard_constraints,son2023enhanced, song24admmpinn,hu2025conditionally}.  For instance, hPINNs~\cite{lu2021physics_hard_constraints} use penalty and augmented-Lagrangian mechanisms to impose hard constraints in inverse-design problems, where the design variables and the PDE solution are learned through neural parameterizations. AL-PINNs~\cite{son2023enhanced} further formulate the PINN training problem by treating initial and boundary conditions as constraints and optimizing a sequence of augmented-Lagrangian objectives. In such methods, the training objective typically takes the form of the PDE residual augmented by multiplier terms and quadratic penalties for constraint violations. The multipliers and penalty parameters are updated over training so that violated constraints are emphasized more strongly. More recently, conditionally adaptive AL methods~\cite{hu2025conditionally} introduce separate penalty parameters for heterogeneous constraints and update them conditionally, so that difficult constraints receive stronger enforcement. These AL-based approaches provide an important step beyond fixed penalty weighting, since the effective weights on the constraints are adapted through dual and penalty updates. However, their practical behavior still depends on the design of penalty schedules, multiplier updates, inner-loop training accuracy, and stopping criteria for the sequential augmented problems. In contrast, AdamFLIP does not enforce constraints by repeatedly increasing penalty weights or performing dual ascent on an augmented objective. Instead, the multiplier in AdamFLIP is computed from the local Jacobian geometry and the desired feedback dynamics, rather than learned through an outer penalty/multiplier update loop. 

A second constrained-optimization direction is based on SQP and trust-region methods. For example, trSQP-PINN~\cite{cheng2024physics} formulates PINN training as a hard-constrained problem and, at each iteration, solves a local approximation obtained by linearizing the constraints and using a quadratic model of the objective within a trust region. To reduce the cost of classical SQP, it uses quasi-Newton approximations and a pretraining stage to improve initialization. For instance, complex constraint implementations like trSQP-PINN require a pre-trained model and use quasi-Newton update to enforce the constraints, which  doubling or tripling the total training time (as shown in our experiments in WT comparison in Table~\ref{tab:combined_results}). 
AdamFLIP is related in spirit to SQP because both methods use local constraint linearization; see Remark~\ref{rmk:connection-with-SQP}. However, their algorithmic structures are different. trSQP-PINN solves a constrained quadratic/trust-region subproblem at each iteration and requires quasi-Newton updates, trust-region management, and pretraining. AdamFLIP instead computes a feedback-linearization multiplier in closed form and then applies an Adam-style adaptive update to the resulting Lagrangian gradient. Thus, AdamFLIP retains local constraint awareness while remaining closer to a first-order neural-network optimizer.

Rather than enforcing constraints through the optimizer, a third class of methods builds them directly into the network architecture. This idea has early roots in trial-solution constructions~\cite{lagaris1998artificial}, where the neural ansatz is designed to satisfy prescribed initial or boundary conditions by construction. More recent work develops this principle in more systematic and problem-adapted ways, including physics-informed architectures for dynamical systems with built-in structural constraints~\cite{djeumou2022neural}, Fourier-based parameterizations that hard-impose certain boundary conditions~\cite{straub2025hard,LI202460}, and KKT-hPINN~\cite{chen2024physics}, which uses a KKT-derived projection layer to enforce linear equality constraints exactly. These hard-constraint-by-construction methods can guarantee exact satisfaction of eligible constraints at every forward pass. However, they are often problem-specific: the ansatz, embedding, or projection layer must be tailored to the boundary condition, domain geometry, or constraint structure. Moreover, they do not directly address the optimization difficulty associated with the remaining PDE residual, data loss, or inverse-parameter identification objective. AdamFLIP is complementary to this line of work: it does not require a specialized architecture, but instead operates at the optimizer level once the relevant PINN losses are formulated as equality constraints.


\begin{remark}[Choice of constrained PINN baselines]
\label{rmk:baseline-choice}
In Section~\ref{sec:results}, we compare AdamFLIP with AL-PINN and trSQP-PINN as two representative optimizer-level constrained PINN baselines. AL-PINN represents augmented-Lagrangian methods, which enforce constraints through multiplier and penalty updates, while trSQP-PINN represents SQP/trust-region methods, which enforce constraints by solving local constrained subproblems. These two methods are therefore the closest classical constrained-optimization counterparts to AdamFLIP.

We do not use architecture-based hard-constraint methods as primary baselines because they change the hypothesis class through problem-specific trial functions, embeddings, or projection layers. Such comparisons would mix the effect of the optimizer with the effect of a specialized constrained parameterization. We also include standard PINN with Adam as the soft-penalty baseline and FL-PINN as an ablation of AdamFLIP without Adam-style adaptive moments.
\end{remark}
\section{Discussion on the Limitations}
While AdamFLIP shows strong empirical performance, several limitations remain. First, our convergence analysis applies to a variant of AdamFLIP rather than the original AdamFLIP used in the experiments. The original algorithm computes the feedback-linearization multiplier in the Euclidean metric but applies an Adam-style preconditioned update, creating a metric mismatch that complicates the analysis. Second, the analysis assumes that $J_h(\theta)$ has full row rank, which may fail in PINN training, especially near exact constraint satisfaction. We mitigate possible ill-conditioning in practice by using the damped solve $J_h(\theta_t)J_h(\theta_t)^\top+\epsilon I$, but extending the theory to rank-deficient or relaxed constraint settings remains open. Finally, AdamFLIP requires choosing which losses are treated as objectives and which are imposed as constraints. Our choices work well for the tested benchmarks, but other PDEs, noisy or sparse data, and higher-dimensional problems may require different formulations and further empirical validation.

\section{Detailed Experimental Setup and Additional Results}
\label{app:detailed_results}
This appendix provides the full experimental details for the numerical results reported in Section~\ref{sec:results}. We summarize the shared experimental settings below; problem-specific settings (architecture, collocation points, inverse problem parameterization) are given in each subsequent subsection.

\subsection{Physics-Informed Neural Networks (PINNs)}
\paragraph{PINN architecture.}
Figure~\ref{fig:pinn-illustration} illustrates the PINN framework used throughout our experiments. The neural network takes spatiotemporal coordinates as input, processes them through multiple fully connected hidden layers with $\tanh$ activations, and outputs the predicted solution field $\hat{u}$. For inverse problems, the unknown PDE parameters (e.g., $\kappa_1$, $\kappa_2$, $\alpha$) are included as learnable scalars in the computational graph and co-optimized with the network weights $\theta$. All required spatial and temporal derivatives are computed via automatic differentiation.

\begin{figure}[htbp]
    \centering
    \includegraphics[width=.9\linewidth]{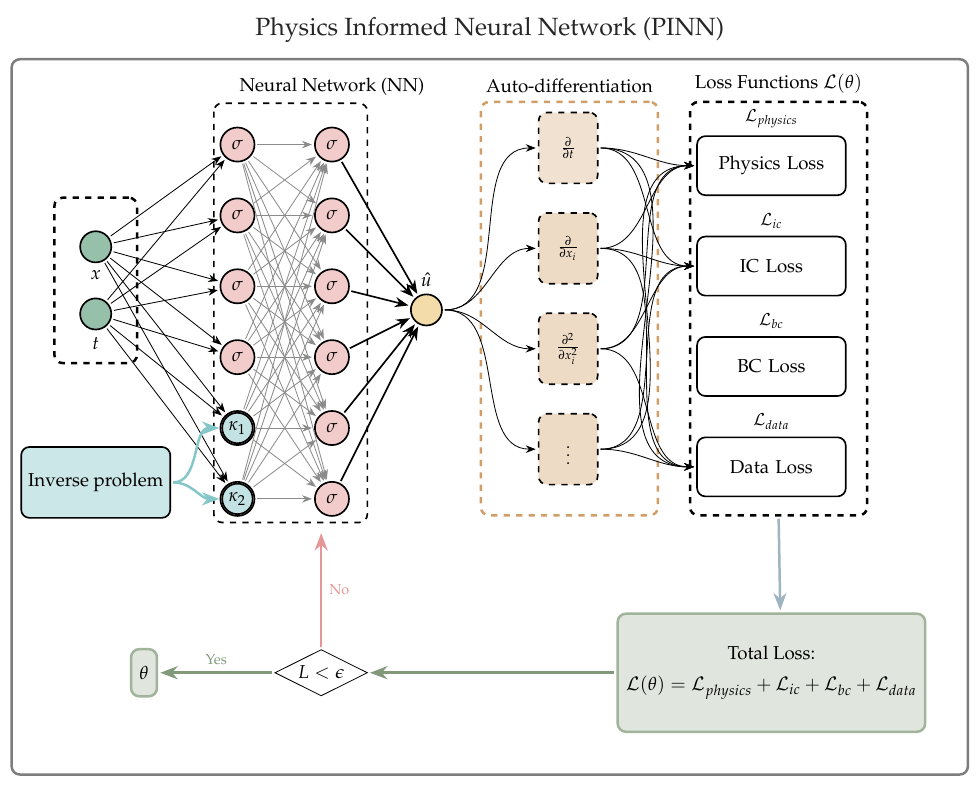}
    \caption{Illustration of the PINN framework. The network maps spatiotemporal inputs to the solution field $\hat{u}$; automatic differentiation provides exact derivatives for constructing the physics, initial condition, and boundary condition losses. In inverse problems, unknown PDE parameters are jointly optimized.}
    \label{fig:pinn-illustration}
\end{figure}

\paragraph{Training procedure.}
For each test problem, all five methods (Standard PINN, AL-PINN, trSQP-PINN, FL-PINN, and AdamFLIP) use an identical network architecture, the same training data, and the same stopping criterion: training terminates when the total loss falls below a predefined threshold $\epsilon$, or a maximum iteration count is reached. The Adam optimizer~\cite{kingma2014adam} is used for the unconstrained baseline (Standard PINN) and as the base optimizer within AdamFLIP, with hyperparameters $\beta_1=0.9$, $\beta_2=0.999$.

\paragraph{Evaluation metrics.}
We report: (i) the relative $L^2$ error between the predicted and analytical solutions on a held-out test grid; (ii) the constraint errors (CE), computed as the mean squared error (MSE) at the initial condition (IC), boundary condition (BC), physics residual (Phy), and---for inverse problems---data loss (DL) collocation points; (iii) the absolute parameter recovery error for inverse problems; and (iv) the training wall-clock time (WT), measured on an HPC node equipped with an NVIDIA H100 GPU (80\,GB).

\paragraph{Choice of Objectives and Constraints in the Algorithm}

For the forward problem, we choose the boundary loss and initial loss as constraints because the boundary and initial conditions define the admissible solution space of the PDE. If these conditions are not well satisfied, the learned solution may be physically inconsistent even when the residual loss is small inside the domain. Therefore, we optimize the physics residual as the main objective while enforcing the boundary and initial losses as constraints. This design encourages the model to reduce the PDE residual without violating the prescribed boundary and initial behaviors.

For the inverse problem, we use the data loss as the objective and treat the physics, boundary, and initial losses as constraints. This is because the main goal of the inverse problem is to fit the observed data and recover the unknown quantities, while the PDE residual and prescribed conditions should still be satisfied to ensure physical consistency. In this setting, the data loss drives parameter identification, and the constraints prevent the learned solution from deviating from the governing equation and known problem conditions.

We performed experiments where all loss terms were treated as constraints with the target loss value set to zero. The final experimental results show that the above objective--constraint choices provide the best overall performance, giving a better balance between accuracy, constraint satisfaction, and training stability.

\paragraph{Constraint Error (CE)} To evaluate the physical consistency of our learned solutions, we adopt the constraint error metric introduced by \cite{utkarsh2025physics}. Specifically, we compute the $\ell_2$ norm of the residuals evaluated by the constraint functions $\mathcal{R}_*$ on each predicted sample $\hat{u}^{(n)}$, and average these norms across all $N$ samples:

\[
\text{CE}(*) = \frac{1}{N} \sum_{n=1}^N \left\| \mathcal{R}_* \left( \hat{u}^{(n)} \right) \right\|_2, \quad * \in \{\text{IC}, \text{BC}, \text{Phy}\}.
\]

Here, the residual functions $\mathcal{R}_*$ quantify the violations of the prescribed Dirichlet initial and boundary conditions, as well as conservation laws.

\paragraph{Collocation sampling.}
Training collocation points are sampled uniformly at random from the respective domains: $\mathcal{D}_{\text{ic}}$ at $t=0$, $\mathcal{D}_{\text{bc}}$ along the spatial boundaries over time, and $\mathcal{D}_{\text{phy}}$ in the spatiotemporal interior. The specific numbers of collocation points ($N_{\text{ic}}$, $N_{\text{bc}}$, $N_f$) vary by problem and are detailed in each subsection below.

\subsection{One-dimensional Burgers' Equation \label{append:burgers}}
The one-dimensional viscous Burgers' equation is a nonlinear partial differential equation frequently used as a benchmark \cite{karniadakis2021physics}. The Burgers' equation with Dirichlet boundary conditions on the spatial domain $\Omega = [-1,1]$ and time domain $[0,T]$ yields the following:
\begin{align}
&\frac{\partial u}{\partial t} - \nu \frac{\partial^2 u}{\partial x^2} + u \frac{\partial u}{\partial x} = 0,\quad x \in \Omega,\; t \in [0,T],\label{eq:burgers}\\[6pt]
&u(x,t) = 0,\quad \forall x \in \partial\Omega,\\[6pt]
&u(x,0) = -\sin(\pi x).
\end{align}
Here, $u(x,t)$ represents the solution over space and time, and the viscosity $\nu$ is chosen to be $0.01/\pi$. The benchmark solution clearly displays the formation and evolution of steep gradients over time.

\subsubsection{Numerical Settings}
\paragraph{Neural network architecture} 
To ensure a fair and consistent comparison across all different models, we employ an identical deep neural network architecture. Specifically, each model utilizes a fully connected network consisting of $8$ hidden layers with $20$ neurons per layer. During the training phase, each model was trained until its loss dropped below the predefined threshold, $\epsilon$.

\paragraph{Training and testing data} 
In the training stage, the IC and BC losses are evaluated using data sampled in space at the initial time $t = 0$ and in time along the spatial boundaries ($x = -1$ and $x = 1$). In contrast, the physics loss is computed on $100$ unlabeled collocation points randomly selected from the interior of the spatiotemporal domain. During the testing stage, the evaluation data are sampled over the spatiotemporal domain $\Omega \times [0,T]$, where $\Omega = [0,1]$. The data is generated on a uniform grid with a spatial mesh size of $\Delta x = 0.01$ and a temporal step size of $\Delta t = 0.01$.

\paragraph{Quantitative comparison metric} To quantitatively assess and compare the performance across all tested models, we calculate the mean squared error (MSE) between the benchmark solution and the model prediction.
\paragraph{Acronyms of all different models}~\\

\begin{table}[H]
\centering
\caption{Acronyms of all different models in numerical results.}
\label{tab:model_acronyms}
\small
\begin{tabular}{llc}
\toprule
{Model Acronym} & {Full name} & Reference\\
\midrule
PINN & Physics-Informed Neural Network&\cite{raissi2019physics} \\
AL-PINN &PINN With Augmented Lagrangian Relaxation Method & \cite{son2023enhanced}  \\
trSQP-PINN & PINN With Trust-Region Sequential Quadratic Programming &\cite{cheng2024physics}  \\
FL-PINN & Feedback linearization PINN & n/a\\
AdamFLIP & Adam feedback linearization PINN &n/a\\
\bottomrule
\end{tabular}
\end{table}

\subsubsection{Forward Problem \label{sec:burgers-forward}}

In this section, we consider the forward problem setting where the physics loss is fully specified and does not involve any unknown parameters to be inferred. The objective is therefore to approximate the solution accurately while ensuring consistency with the prescribed initial condition, boundary condition, and governing equation.

Table~\ref{tab:burgers_forward_results} presents a quantitative comparison of standard PINN (Adam), AL-PINN, trSQP-PINN, FL-PINN, and AdamFLIP for the Burgers' equation, measured in terms of the relative $L^2$ error and the constraint errors (CE) for the initial condition (IC), boundary condition (BC), and physics loss (PL). Among all models considered, AdamFLIP performs best as achieving a relative $L^2$ error that is approximately one order of magnitude smaller than those of the other models while also yielding the lowest constraint errors (CE) in most cases. Although trSQP-PINN attains the smallest physics-loss CE, which is several orders of magnitude lower than those of the other models, this advantage does not translate into improved relative $L^2$ error accuracy or better performance in the other constraint errors. It is also worth noting that all models yield comparable wall-clock times, indicating that the improved performance of AdamFLIP is achieved without introducing additional computational cost.

\begin{table}[h]
\centering
\setlength{\tabcolsep}{4pt} 
\caption{Quantitative error comparison for the 1D Burgers equation (forward problem). We report relative $L^2$ errors, constraint errors (CE) at initial (IC), boundary (BC), and physics (Phy) collocation points, and training wall-clock time (WT).}
\label{tab:burgers_forward_results}
\resizebox{\textwidth}{!}{
\begin{tabular}{lcccccc}
\toprule
\textbf{Method} & \textbf{Constraint} & \textbf{Rel. $L^2$} & \textbf{CE (IC)} & \textbf{CE (BC)} & \textbf{CE (Phy)} & \textbf{WT (s)} \\
\midrule
Standard PINN & \xmark  
& $2.51\times10^{-1}$ 
& $4.26\times10^{-3}$ 
& $2.44\times10^{-3}$ 
& $1.49\times10^{-3}$ 
& \textbf{113} \\

AL-PINN & \cmark  
& $8.20\times10^{-2}$ 
& $2.34\times10^{-3}$ 
& $6.25\times10^{-5}$ 
& $2.97\times10^{-3}$ 
& 155 \\

trSQP-PINN & \cmark  
& $4.43\times10^{-1}$ 
& $9.91\times10^{-2}$ 
& $9.80\times10^{-2}$ 
& $\mathbf{1.31\times10^{-3}}$ 
& 625 \\

FL-PINN & \cmark  
& $1.06\times10^{-1}$ 
& $6.82\times10^{-4}$ 
& $4.40\times10^{-4}$ 
& $1.13\times10^{-2}$ 
& 133 \\

\textbf{AdamFLIP} & \cmark  
& $\mathbf{3.12\times10^{-2}}$ 
& $\mathbf{1.00\times10^{-4}}$ 
& $\mathbf{1.00\times10^{-4}}$ 
& $2.60\times10^{-3}$ 
& 115 \\
\bottomrule
\end{tabular}%
}
\end{table}
\raggedbottom

\begin{figure}[H]
    \centering
    \begin{subfigure}[b]{\linewidth}
        \centering
        \includegraphics[width=\linewidth]{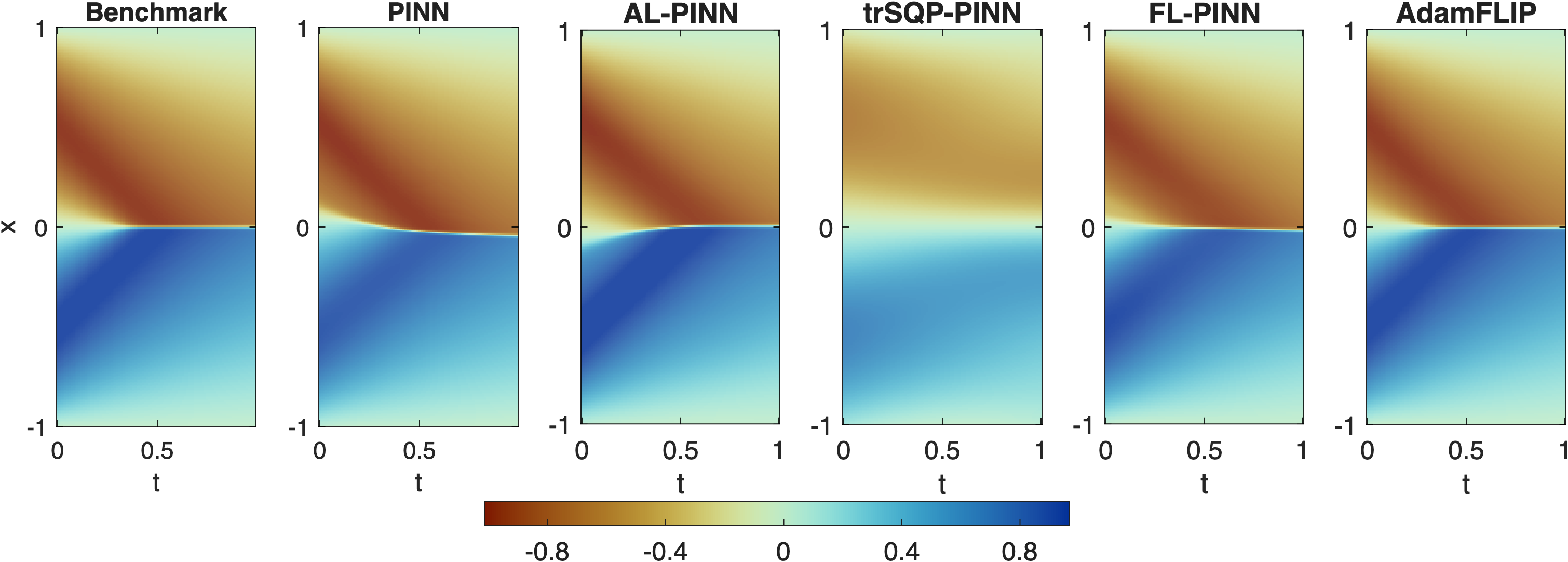}
        \caption{Spatiotemporal fields}
        \label{fig:burgers-forward}
    \end{subfigure}
    \begin{subfigure}[b]{.8\linewidth}
        \centering
        \includegraphics[width=\linewidth]{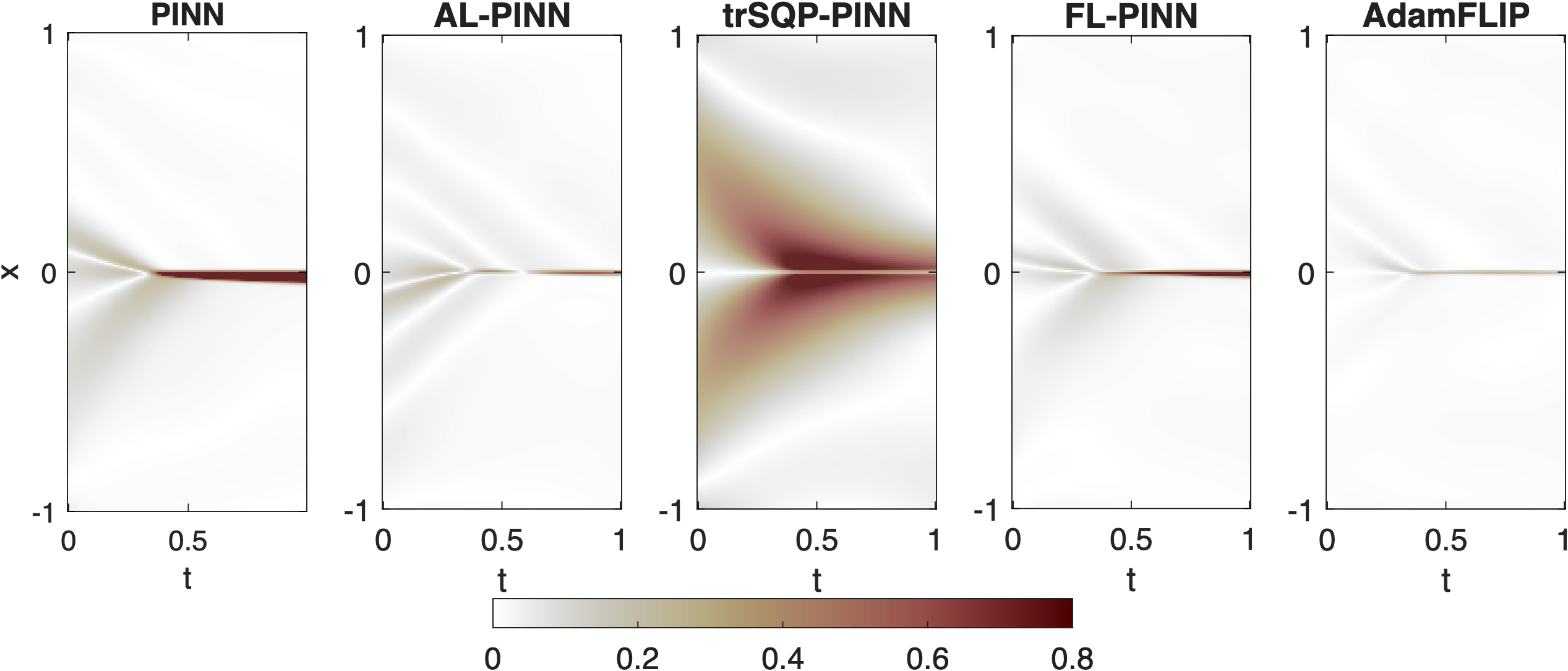}
        \caption{Absolute errors}
        \label{fig:burgers-forward-error}
    \end{subfigure}
    \caption{Comparison of spatiotemporal solutions and absolute errors for the one-dimensional Burgers' equation. (a) Predicted fields $u(x,t)$ from the benchmark solution and from Standard PINN, FL-PINN, trSQP-PINN, and AdamFLIP. (b) Corresponding absolute error distributions for each method. }
    \label{fig:forward-burgers}
\end{figure}

This is also confirmed in Figure \ref{fig:forward-burgers}. In particular, Figure~\ref{fig:burgers-forward} compares the spatiotemporal fields predicted by different models with the benchmark solution (first panel). Compared with the benchmark solution, the standard PINN (second panel) captures the overall spatiotemporal structure but exhibits noticeable smoothing in regions that should contain sharp gradients. The FL-PINN (third panel) improves the representation of these features by more accurately capturing the solution structure along the characteristic directions of the dynamics, while AdamFLIP (forth panel) performs best overall, effectively preserving both the magnitude and the sharpness of the solution profile.

Figure~\ref{fig:burgers-forward-error} further supports these observations by presenting the spatiotemporal absolute error distributions for the different models.It is clearly observed that AdamFLIP attains the smallest error throughout the entire spatiotemporal domain, with substantially smaller errors in the shock regions than the other models. FL-PINN reduces these errors by incorporating feedback linearization, yet residual discrepancies remain visible. In contrast, AdamFLIP achieves the lowest absolute error throughout the domain, with errors significantly suppressed in both magnitude and spatial extent. These observations are consistent with the quantitative results in Table~\ref{tab:burgers_forward_results}, confirming that AdamFLIP provides superior accuracy while enforcing physical constraints more effectively.

\subsubsection{Inverse Problem \label{sec:burgers-inverse}}
In the inverse problem setting, we assume that there are two unknown PDE parameters $\kappa_1$ and $\kappa_2$ in Eq.~\eqref{eq:burgers} , which yields the following setting: 
\begin{align}
&\frac{\partial u}{\partial t}  + \kappa_1 u\frac{\partial u}{\partial x} =  \kappa_2 \frac{\partial^2 u}{\partial x^2},\quad x \in \Omega,\; t \in [0,T], \label{eq:burgers-inverse-1}\\[6pt]
&u(x,t) = 0,\quad \forall x \in \partial\Omega,\label{eq:burgers-inverse-2}\\[6pt]
&u(x,0) = -\sin(\pi x).\label{eq:burgers-inverse-3}
\end{align}
Within the PINN frameworks, $\kappa_1$ and $\kappa_2$ are parameterized as learnable scalars, which are co-optimized with the neural network weights by minimizing the loss function same as in the forward problem. In this setup, for all models we initialize the parameter estimates at $\kappa_1 =2,\kappa_2=0$, whereas their ground truth values are $\kappa_1 = 1 $ and$\kappa_2= \nu = 0.01/{\pi} \approx 0.00318$.

Table~\ref{tab:inverse_params} shows the learned PDE parameters and their $L^1$ error compared to the ground truth for the Burgers' equation in the inverse problem setting. It is shown that among all models, AdamFLIP demonstrates the highest accuracy by achieving the lowest $L^1$ errors in estimating both $\kappa_1$ and $\kappa_2$. Specifically, it reduces the estimation error of $\kappa_1$ by nearly an order of magnitude compared to the next best model, FL-PINN, while simultaneously improving the estimation accuracy for $\kappa_2$ by at least one-third relative to the other methods. We use the same neural network structure same as used in forward problem but treat the unknown parameters as the trainable parameters in the neural network. Notably, trSQP-PINN exhibits the largest error in $\kappa_1$ ($0.9821$), indicating unstable or inaccurate parameter identification. By enforcing parameter bounds and adaptively balancing constraints, AdamFLIP substantially improves robustness and accuracy in parameter discovery.

\begin{table}[h]
\centering
\caption{Parameter estimation results for the one-dimensional Burgers' equation in the inverse problem setting. The exact PDE parameters are $\kappa_1=1$ and $\kappa_2=0.003183$.}
\label{tab:inverse_params}
\begin{tabular}{lcccc}
\toprule
\textbf{Method} & \textbf{$\hat{\kappa}_1$} & \textbf{$|\hat{\kappa}_1-\kappa_1|$} & \textbf{$\hat{\kappa}_2$} & \textbf{$|\hat{\kappa}_2-\kappa_2|$} \\
\midrule
Standard PINN & $0.9667$ & $0.0333$ & $0.0097$ & $0.0065$ \\
AL-PINN & $1.5940$ & $0.5940$ & $0.0829$ & $0.0797$ \\
trSQP-PINN & $1.9821$ & $0.9821$ & $0.0222$ & $0.0190$ \\
FL-PINN & $0.8723$ & $0.1277$ & $0.0996$ & $0.0965$ \\
\textbf{AdamFLIP} & $\mathbf{1.0154}$ & $\mathbf{0.0154}$ & $\mathbf{0.0067}$ & $\mathbf{0.0036}$ \\
\bottomrule
\end{tabular}
\end{table}

Similar to Table~\ref{tab:burgers_forward_results}, Table \ref{tab:inverse_results} presents a quantitative comparison of standard PINN (Adam), AL-PINN, trSQP-PINN, FL-PINN, and AdamFLIP for the Burgers' equation in the inverse problem setting, where the relative $L^2$ error and the constraint errors (CE) for the initial condition (IC), boundary condition (BC), and physics loss (PL) are compared. Consistent with the forward problem setting, AdamFLIP performs best, achieving a relative $L^2$ error that is approximately one order of magnitude smaller than those of the other models. Furthermore, it yields the lowest constraint errors in most categories. For example, while AL-PINN and FL-PINN manage to reduce the data loss constraint error, CE(DL), to roughly half of that produced by the standard PINN, AdamFLIP minimizes this error by nearly two orders of magnitude relative to the baseline. Similarly, AdamFLIP reduces the initial and boundary constraint errors to approximately one-tenth of those observed in the AL-PINN and FL-PINN models. While constrained models incur overhead in computational cost, i.e., AL-PINN, FL-PINN and trSQP-PINN take about one-third longer to train than standard PINN, AdamFLIP remains relative efficient.

Consistent with the quantitative results in Tables~\ref{tab:inverse_params}--\ref{tab:inverse_results}, Figure~\ref{fig:inverse-burgers} presents a qualitative comparison of the reconstructed spatiotemporal solutions and their corresponding absolute errors. The benchmark solution is characterized by a sharp transition near $x=0$, a feature that AdamFLIP accurately captures across the entire temporal domain. In contrast, the standard PINN fails to resolve the underlying physical structure, instead converging to a trivial, highly diffuse profile that neglects the sharp interface. While FL-PINN improves upon the baseline, it still exhibits noticeable discrepancies near the transition region. These differences are further emphasized in the absolute error plots: the standard PINN shows significant error concentrations around the discontinuity, and FL-PINN achieves only moderate error reduction, whereas AdamFLIP maintains uniformly low errors throughout the entire domain. Overall, AdamFLIP proves to be the most effective framework, significantly outperforming other models with or without constraints by accurately identifying the governing physics while reconstructing the underlying spatiotemporal dynamics with moderate computational overhead.

\begin{table}[h]
\centering
\setlength{\tabcolsep}{4pt} 
\caption{Quantitative error comparison for the 1D Burgers equation (inverse problem). We report relative $L^2$ errors, constraint errors (CE) at initial (IC), boundary (BC), physics (Phy), and data loss (DL) collocation points, and training wall-clock time (WT).}
\label{tab:inverse_results}
\resizebox{\textwidth}{!}{
\begin{tabular}{lccccccc}
\toprule
\textbf{Method} & \textbf{Constraint} & \textbf{Rel. $L_2$} & \textbf{CE (IC)} & \textbf{CE (BC)} & \textbf{CE (Phy)} & \textbf{CE (DL)} & \textbf{WT (s)} \\
\midrule
Standard PINN & \xmark 
& $5.46\times10^{-1}$ 
& $9.00 \times 10^{-3}$ 
& $\mathbf{1.00\times10^{-4}}$ 
& $\mathbf{1.00\times10^{-4}}$ 
& $1.13\times10^{-1}$ 
& 248 \\

AL-PINN & \cmark 
& $4.02\times10^{-1}$ 
& $3.06\times10^{-2}$ 
& $2.73\times10^{-2}$ 
& $2.48\times10^{-2}$ 
& $5.94\times10^{-2}$ 
& 334 \\

trSQP-PINN & \cmark 
& $4.54\times10^{-1}$ 
& $4.41\times10^{-1}$ 
& $2.56\times10^{-1}$ 
& $5.20\times10^{-3}$ 
& $4.07\times10^{-1}$ 
& 621 \\

FL-PINN & \cmark 
& $4.19\times10^{-1}$ 
& $2.26\times10^{-2}$ 
& $2.04\times10^{-2}$ 
& $1.80\times10^{-3}$ 
& $6.62\times10^{-2}$ 
& 334 \\

\textbf{AdamFLIP} & \cmark 
& $\mathbf{5.74\times10^{-2}}$ 
& $\mathbf{2.90\times10^{-3}}$ 
& $3.50\times10^{-3}$ 
& $2.11\times10^{-2}$ 
& $\mathbf{1.20\times10^{-3}}$ 
& 324 \\
\bottomrule
\end{tabular}%
}
\end{table}

\begin{figure}[H]
    \centering
    \begin{subfigure}[b]{\linewidth}
        \centering
        \includegraphics[width=\linewidth]{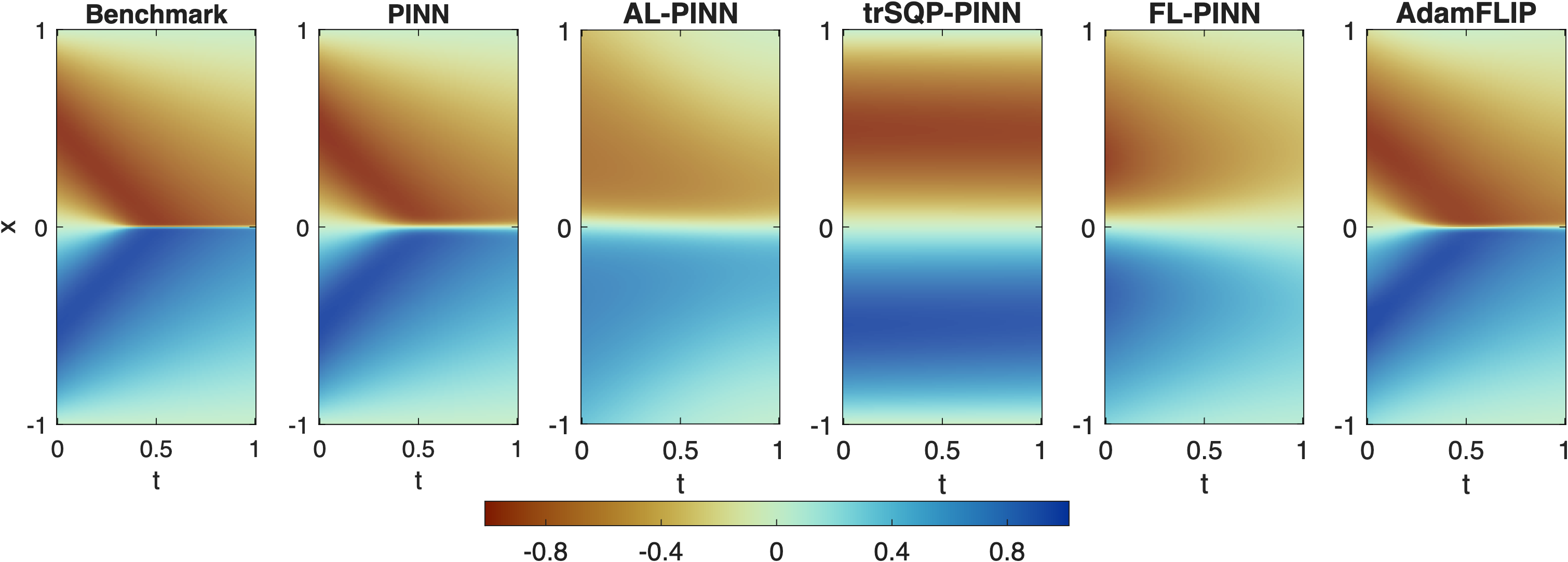}
        \caption{Spatiotemporal fields}
        \label{fig:burgers-inverse}
    \end{subfigure}
    \begin{subfigure}[b]{.8\linewidth}
        \centering
        \includegraphics[width=\linewidth]{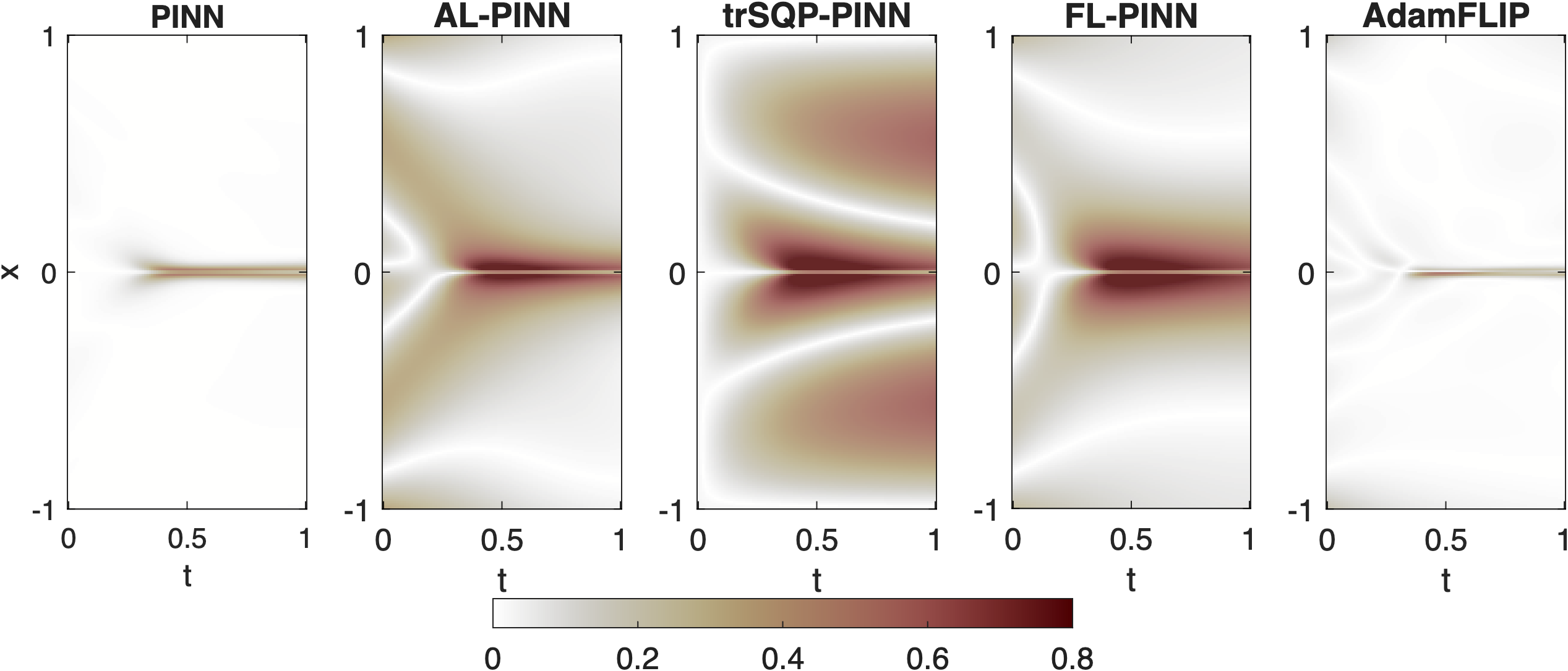}
        \caption{Absolute errors}
        \label{fig:burgers-inverse-error}
    \end{subfigure}
    \caption{Comparison of spatiotemporal solutions and absolute errors for the inverse one-dimensional Burgers' equation. (a) Predicted fields $u(x,t)$ from the benchmark solution and from Standard PINN, FL-PINN, trSQP-PINN ,and AdamFLIP. (b) Corresponding absolute error distributions for each method.  }
    \label{fig:inverse-burgers}
\end{figure}

\subsection{Time-fractional Mixed Diffusion-Wave Equations (TFMDWEs)}
The third example we consider is the time-fractional mixed diffusion-wave equation (TFMDWE). This class of equations is widely used to model anomalous transport processes and memory effects in complex media, such as viscoelastic materials, where the dynamics transition between pure diffusion and wave propagation \cite{luchko2011fractional}. Specifically, we consider the tTFMDWEs on the spatial domain $\Omega = (0, \pi)$ with homogeneous Dirichlet boundary conditions:
\begin{align}
& D_t^{\alpha} u(x,t) = \frac{\partial^2 u}{\partial x^2} + f(x,t), \quad t \in [0,1],\; x \in \Omega \equiv [0,\pi],\\[6pt]
&u(x,t) = 0,\quad \forall x \in \partial\Omega,\\[6pt]
& u(x, 0) = 0, \quad x \in \Omega,
\end{align}
where $ D_t^{\alpha}$ denotes the Caputo fractional derivative of order $\alpha \in (0,1]$. In addition, we prescribe the forcing term as:
\begin{equation}
f(x,t) = \left[ \frac{\Gamma(4)}{\Gamma(4-\alpha)} t^{3-\alpha} + t^3 \right] \sin(x),
\end{equation}
which corresponds to the analytical solution $u(x,t) = t^3 \sin(x)$. Here, $\Gamma(\cdot)$ represents the Euler gamma function.

\subsubsection{PINN's settings}
\paragraph{Neural network architecture} 
To ensure a fair and consistent comparison across all different models, we employ an identical deep neural network architecture. Specifically, each model utilizes a fully connected network consisting of $3$ hidden layers with $100$ neurons per layer. During the training phase, each model was trained until its loss dropped below the predefined threshold, $\epsilon$.

\paragraph{Training and testing data} 
In the training stage, the initial condition (IC) and boundary condition (BC) losses are evaluated using data uniformly sampled in space at $t = 0$ and in time along the spatial boundaries ($x = 0$ and $x = \pi$), with $N_{\text{ic}} = 100$ and $N_{\text{bc}} = 100$ points, respectively. In contrast, the physics loss for the TFMDWE is computed on $N_f = 5000$ interior collocation points randomly selected from the spatiotemporal domain $(0,\pi) \times (0,1]$. Additionally, $N_{\text{data}} = 100$ synthetic observation points are sampled from the exact solution $u(x,t) = t^3 \sin(x)$ to inform the training process. During the testing stage, the evaluation data are sampled over the spatiotemporal domain $\Omega \times [0,T]$, where $\Omega = [0,\pi]$ and $T=1$. The data is generated on a uniform grid with a spatial mesh size of $\Delta x = \pi/100$ and a temporal step size of $\Delta t = 0.01$, resulting in a $100 \times 100$ grid for quantitative performance assessment.

\subsubsection{Forward}
Table~\ref{tab:fpde_forward_results} reports the numerical performance of Standard PINN, AL-PINN, trSQP-PINN, FL-PINN, and AdamFLIP on the forward fractional PDE problem. Overall, AdamFLIP provides the best accuracy and constraint satisfaction, achieving the lowest relative $L_2$ error of $2.85 \times 10^{-1}$ while substantially reducing the constraint errors to $\mathrm{CE(IC)} = 1.42 \times 10^{-5}$ and $\mathrm{CE(BC)} = 1.47 \times 10^{-4}$. In comparison, FL-PINN improves over the baselines with a relative error of $4.26 \times 10^{-1}$, but its constraint errors remain larger ($\mathrm{CE(IC)} = 1.36 \times 10^{-4}$, $\mathrm{CE(BC)} = 4.20 \times 10^{-3}$) and its physics-loss constraint error is higher than AdamFLIP. Both AL-PINN and trSQP-PINN reduce the relative error compared to Standard PINN, yet they exhibit markedly poorer physics-loss enforcement, with $\mathrm{CE(PL)}$ on the order of $10^{-1}$--$10^{0}$, indicating unstable or imbalanced training. Standard PINN performs worst overall, yielding a large relative error of $7.71 \times 10^{-1}$ with comparatively weak constraint satisfaction. Moreover, AdamFLIP achieves these improvements with the lowest wall-clock time ($50\,\mathrm{s}$), demonstrating a favorable trade-off between accuracy, constraint enforcement, and computational efficiency for forward fractional PDE learning.

\begin{table}[h]
\centering
\setlength{\tabcolsep}{4pt} 
\caption{Quantitative error comparison for the forward fractional PDE. We report relative $L^2$ errors, constraint errors (CE) at initial (IC), boundary (BC), and physics (Phy) collocation points, and training wall-clock time (WT).}
\label{tab:fpde_forward_results}
\resizebox{\textwidth}{!}{
\begin{tabular}{lcccccc}
\toprule
\textbf{Method} & \textbf{Constraint} & \textbf{Rel. $L_2$} & \textbf{CE (IC)} & \textbf{CE (BC)} & \textbf{CE (Phy)} & \textbf{WT (s)} \\
\midrule
Standard PINN & \xmark 
& $7.71\times10^{-1}$ 
& $1.65\times10^{-3}$ 
& $2.19\times10^{-2}$ 
& $1.07\times10^{-1}$ 
& 56 \\

AL-PINN & \cmark 
& $4.93\times10^{-1}$ 
& $1.11\times10^{-3}$ 
& $5.16\times10^{-2}$ 
& $7.88\times10^{-1}$ 
& 53 \\

trSQP-PINN & \cmark 
& $4.56\times10^{-1}$ 
& $1.69\times10^{-3}$ 
& $2.90\times10^{-2}$ 
& $8.13\times10^{-1}$ 
& 105 \\

FL-PINN & \cmark 
& $4.26\times10^{-1}$ 
& $1.36\times10^{-4}$ 
& $4.20\times10^{-3}$ 
& $1.19\times10^{-1}$ 
& 52 \\

\textbf{AdamFLIP} & \cmark 
& $\mathbf{2.85\times10^{-1}}$ 
& $\mathbf{1.42\times10^{-5}}$ 
& $\mathbf{1.47\times10^{-4}}$ 
& $\mathbf{8.87\times10^{-2}}$ 
& \textbf{50} \\
\bottomrule
\end{tabular}%
}
\end{table}

Figure~\ref{fig:forward-fpde} visualizes the spatiotemporal solution fields and corresponding absolute error maps for the forward one-dimensional time-fractional mixed diffusion--wave equation. As shown in Fig.~7(a), the benchmark solution exhibits a smooth high-amplitude region concentrated near the upper time boundary, with the strongest response occurring around the center of the spatial domain. The Standard PINN prediction in Fig.~7(b) deviates noticeably from the benchmark, particularly in the high-gradient region near the top boundary, where the solution amplitude and spread are misrepresented. This discrepancy is reflected in the error map in Fig.~7(c), where the Standard PINN displays pronounced localized errors concentrated near the upper portion of the domain, indicating difficulty in resolving the sharp spatiotemporal transition governed by the fractional dynamics.

In contrast, FL-PINN produces a solution field that more closely matches the benchmark structure, reducing the mismatch in the dominant high-amplitude region. The corresponding absolute error map shows a substantial reduction in error magnitude and spatial extent compared with the Standard PINN, suggesting that feedback linearization improves optimization stability and accelerates convergence toward a physically consistent solution. AdamFLIP yields the closest agreement with the benchmark among all methods. Its predicted spatiotemporal field preserves both the location and smooth shape of the dominant response region, and its absolute error map exhibits the smallest and most spatially confined residuals. In particular, the remaining errors are primarily limited to a narrow band near the upper time boundary, consistent with a challenging region where the solution gradients are largest and where PINN training typically struggles. Overall, Fig.~\ref{fig:forward-fpde} demonstrates that adaptive feedback linearization not only improves solution accuracy relative to the Standard PINN but also yields more uniform and reduced errors across the full spatiotemporal domain, aligning with the quantitative improvements reported in Table~\ref{tab:fpde_forward_results}.

\begin{figure}[H]
    \centering
   \begin{subfigure}[b]{0.48\linewidth}
        \centering
        \includegraphics[width=\linewidth]{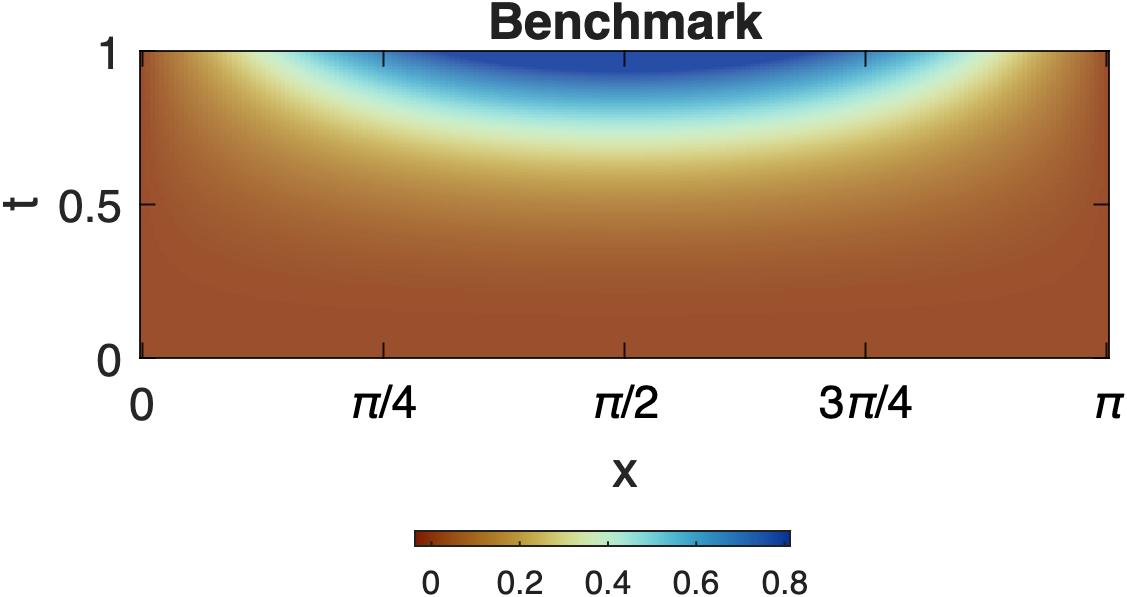}
        \caption{Spatiotemporal field of Benchmark}
        \label{fig:fpde-forward}
    \end{subfigure}
    \\
    \begin{subfigure}[b]{0.48\linewidth}
        \centering
        \includegraphics[width=\linewidth]{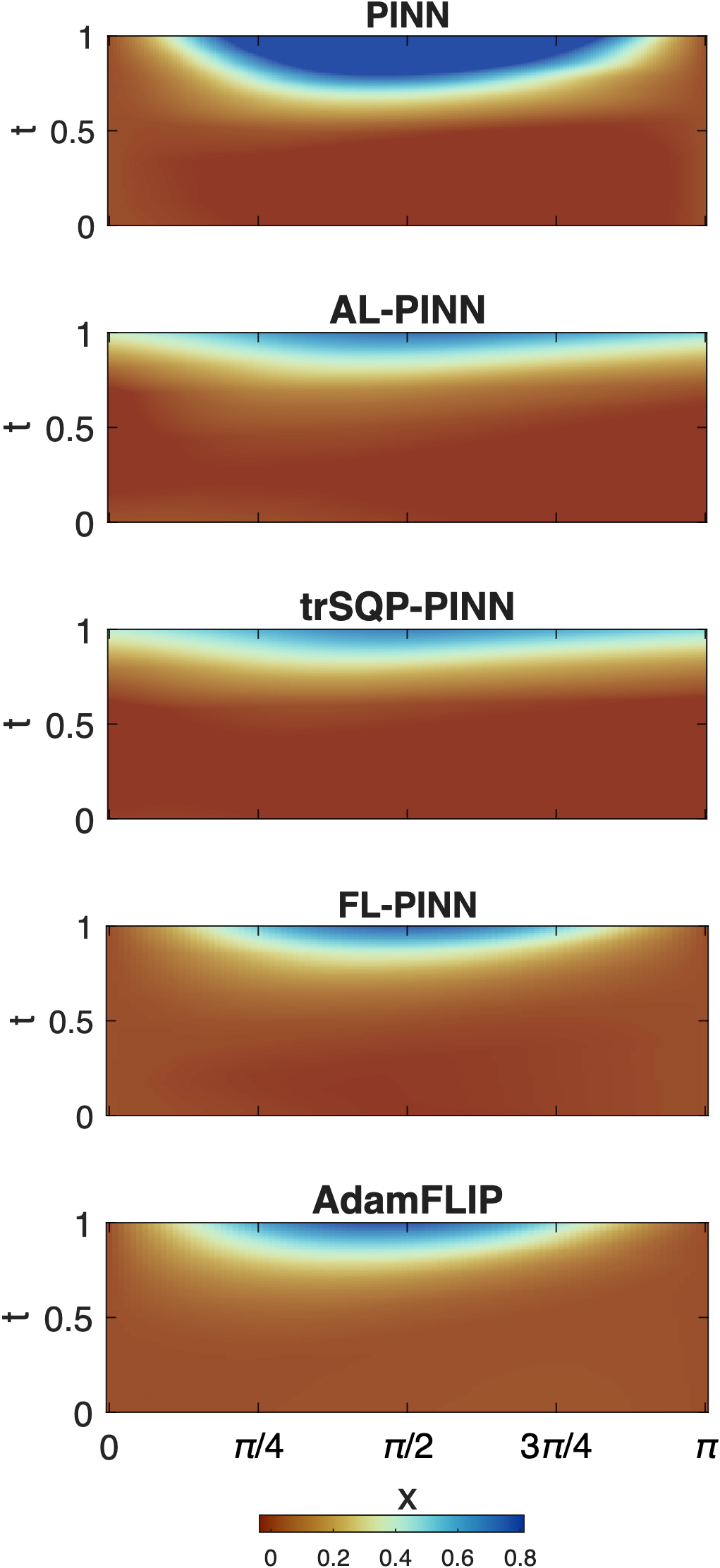}
        \caption{Spatiotemporal fields of different models}
        \label{fig:fpde-forward}
    \end{subfigure}\hfill
    \begin{subfigure}[b]{0.48\linewidth}
        \centering
        \includegraphics[width=\linewidth]{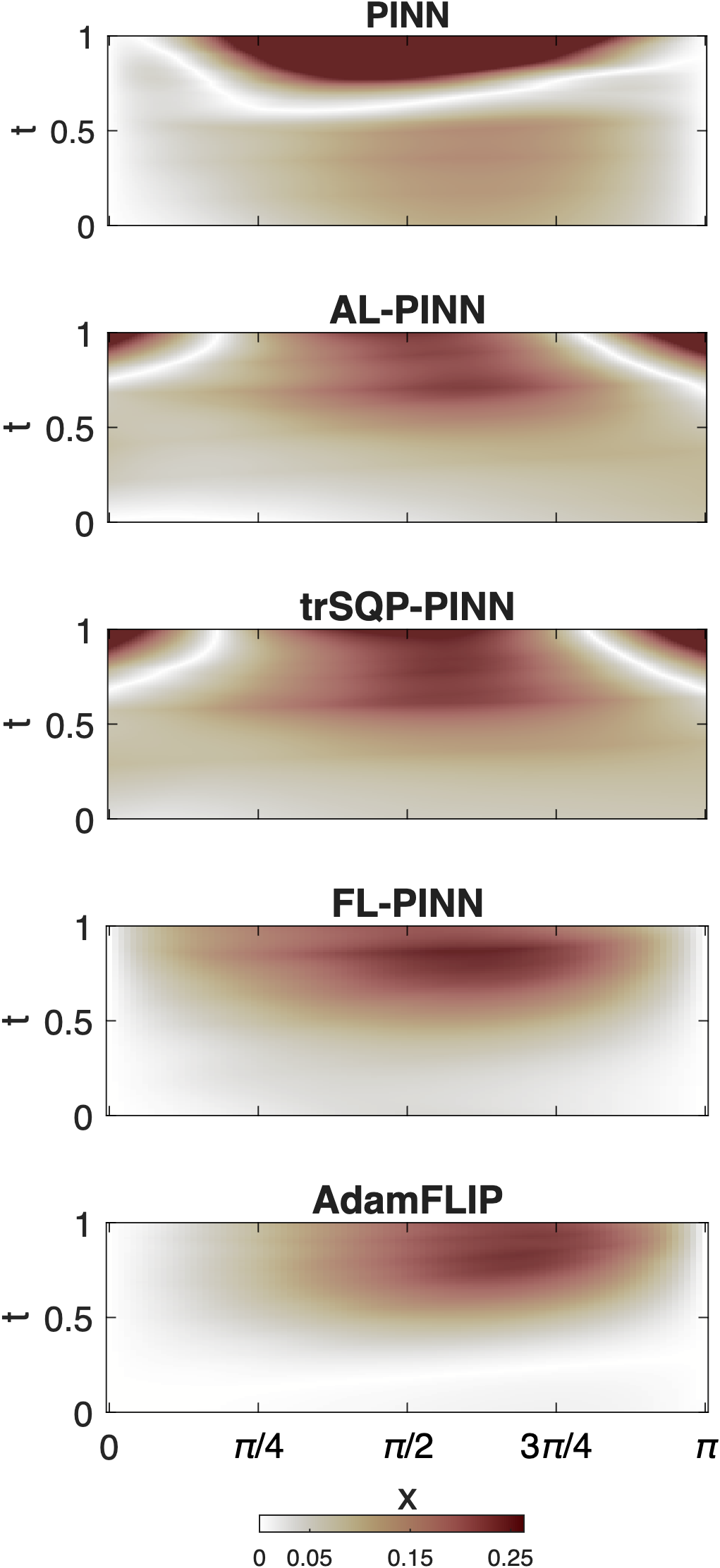}
        \caption{Absolute errors}
        \label{fig:fpde-forward-error}
    \end{subfigure}
    \caption{Forward time-fractional mixed diffusion--wave equation (1D) results: benchmark spatiotemporal field and model predictions. (a) Reference solution $u(x,t)$. (b) Predicted spatiotemporal fields obtained by Standard PINN, FL-PINN, trSQP-PINN, and AdamFLIP. (c) Corresponding absolute error maps , showing that AdamFLIP yields the smallest and most localized errors, especially near the high-gradient region at large $t$.}
    \label{fig:forward-fpde}
\end{figure}

\subsubsection{Inverse}
Table~\ref{tab:fpde_inverse_results} reports a quantitative comparison for the inverse fractional PDE problem across Standard PINN, AL-PINN, trSQP-PINN, FL-PINN, and AdamFLIP in terms of the relative $L_2$ error, constraint errors for the initial condition (IC), boundary condition (BC), physics loss (PL), and data loss (DL), as well as wall-clock time. Overall, AdamFLIP achieves the best performance, attaining the lowest relative $L_2$ error of $3.95 \times 10^{-2}$ and the smallest data loss $\mathrm{CE(DL)} = 1.13 \times 10^{-4}$, indicating substantially improved data fitting and more reliable parameter identification compared with all baselines.

The Standard PINN exhibits moderate solution accuracy ($2.22 \times 10^{-1}$) but suffers from noticeably larger violations of the boundary condition and physics constraints ($\mathrm{CE(BC)} = 2.27 \times 10^{-2}$, $\mathrm{CE(PL)} = 7.01 \times 10^{-2}$). AL-PINN and trSQP-PINN perform poorly in this inverse setting, producing large relative errors on the order of $10^{-1}$--$10^{0}$ and comparatively large data losses, suggesting that they either overemphasize certain constraints or converge to suboptimal parameter values. FL-PINN partially improves constraint enforcement, achieving a small $\mathrm{CE(IC)} = 1.56 \times 10^{-5}$ and $\mathrm{CE(BC)} = 4.08 \times 10^{-4}$, but its overall accuracy remains limited ($5.99 \times 10^{-1}$). This behavior is consistent with a failure mode in inverse optimization where the method collapses to a trivial solution, yielding low constraint residuals while failing to match observational data. In contrast, AdamFLIP adaptively balances constraints and data fitting during training, preserving strong IC/BC enforcement ($\mathrm{CE(IC)} = 1.77 \times 10^{-5}$, $\mathrm{CE(BC)} = 2.57 \times 10^{-4}$) while dramatically reducing the data loss and improving overall solution accuracy, with a comparable computational cost.

\begin{table}[h]
\centering
\setlength{\tabcolsep}{4pt} 
\caption{Quantitative error comparison for the inverse fractional PDE problem. We report relative $L^2$ errors, constraint errors (CE) at initial (IC), boundary (BC), physics (Phy), and data loss (DL) collocation points, and training wall-clock time (WT).}
\label{tab:fpde_inverse_results}
\resizebox{\textwidth}{!}{
\begin{tabular}{lccccccc}
\toprule
\textbf{Method} & \textbf{Constraint} & \textbf{Rel. $L_2$} & \textbf{CE (IC)} & \textbf{CE (BC)} & \textbf{CE (Phy)} & \textbf{CE (DL)} & \textbf{WT (s)} \\
\midrule
Standard PINN & \xmark 
& $2.22\times10^{-1}$ 
& $9.01\times10^{-4}$ 
& $2.27\times10^{-2}$ 
& $7.01\times10^{-2}$ 
& $3.59\times10^{-3}$ 
& \textbf{120} \\

AL-PINN & \cmark 
& $8.81\times10^{-1}$ 
& $2.60\times10^{-4}$ 
& $3.74\times10^{-3}$ 
& $4.92\times10^{-3}$ 
& $5.62\times10^{-2}$ 
& 134 \\

trSQP-PINN & \cmark 
& $6.67\times10^{-1}$ 
& $1.19\times10^{-4}$ 
& $3.59\times10^{-3}$ 
& $\mathbf{3.78\times10^{-3}}$ 
& $5.44\times10^{-2}$ 
& 144 \\

FL-PINN & \cmark 
& $5.99\times10^{-1}$ 
& $\mathbf{1.56\times10^{-5}}$ 
& $4.08\times10^{-4}$ 
& $1.82\times10^{-1}$ 
& $2.61\times10^{-2}$ 
& 141 \\

\textbf{AdamFLIP} & \cmark 
& $\mathbf{3.95\times10^{-2}}$ 
& $1.77\times10^{-5}$ 
& $\mathbf{2.57\times10^{-4}}$ 
& $6.53\times10^{-2}$ 
& $\mathbf{1.13\times10^{-4}}$ 
& 133 \\
\bottomrule
\end{tabular}%
}
\end{table}
Figure~\ref{fig:inverse-fpde} illustrates the spatiotemporal solution fields and corresponding absolute error maps for the inverse one-dimensional time-fractional mixed diffusion--wave equation. Compared with the forward case, the inverse problem is more challenging due to the additional requirement of parameter identification while simultaneously fitting observational data. The Standard PINN solution exhibits noticeable distortion in the upper time region, where the dominant response is expected. The predicted field appears overly smoothed and shifted, leading to large localized errors near the top boundary, as highlighted in the absolute error map. This behavior reflects the tendency of unconstrained optimization to converge to suboptimal parameter values that partially satisfy the PDE residual but fail to accurately reconstruct the underlying dynamics.

FL-PINN improves constraint enforcement and produces a solution with a clearer dominant structure compared to Standard PINN. However, the reconstructed field still deviates from the expected spatiotemporal pattern, particularly in the central region, where the amplitude and curvature are misrepresented. The corresponding error map reveals persistent residual bands, indicating that strict constraint minimization alone does not guarantee accurate parameter recovery in the inverse setting. AdamFLIP yields a solution field that most closely matches the expected physical structure, accurately capturing both the spatial localization and temporal evolution of the dominant response. The absolute error map is nearly uniform and close to zero across the entire domain, with only minor residuals near regions of steep gradient. This demonstrates that adaptive constraint balancing effectively mitigates the degeneracy observed in pure feedback linearization and prevents convergence to trivial or physically inconsistent parameter values. Overall, Fig.~\ref{fig:inverse-fpde} visually confirms the quantitative results reported in Table~\ref{tab:fpde_inverse_results}: AdamFLIP provides substantially improved accuracy, more reliable parameter identification, and significantly reduced global error compared to Standard PINN and FL-PINN in the inverse fractional PDE setting.

\begin{figure}[H]
    \centering
    \begin{subfigure}[b]{0.48\linewidth}
        \centering
        \includegraphics[width=\linewidth]{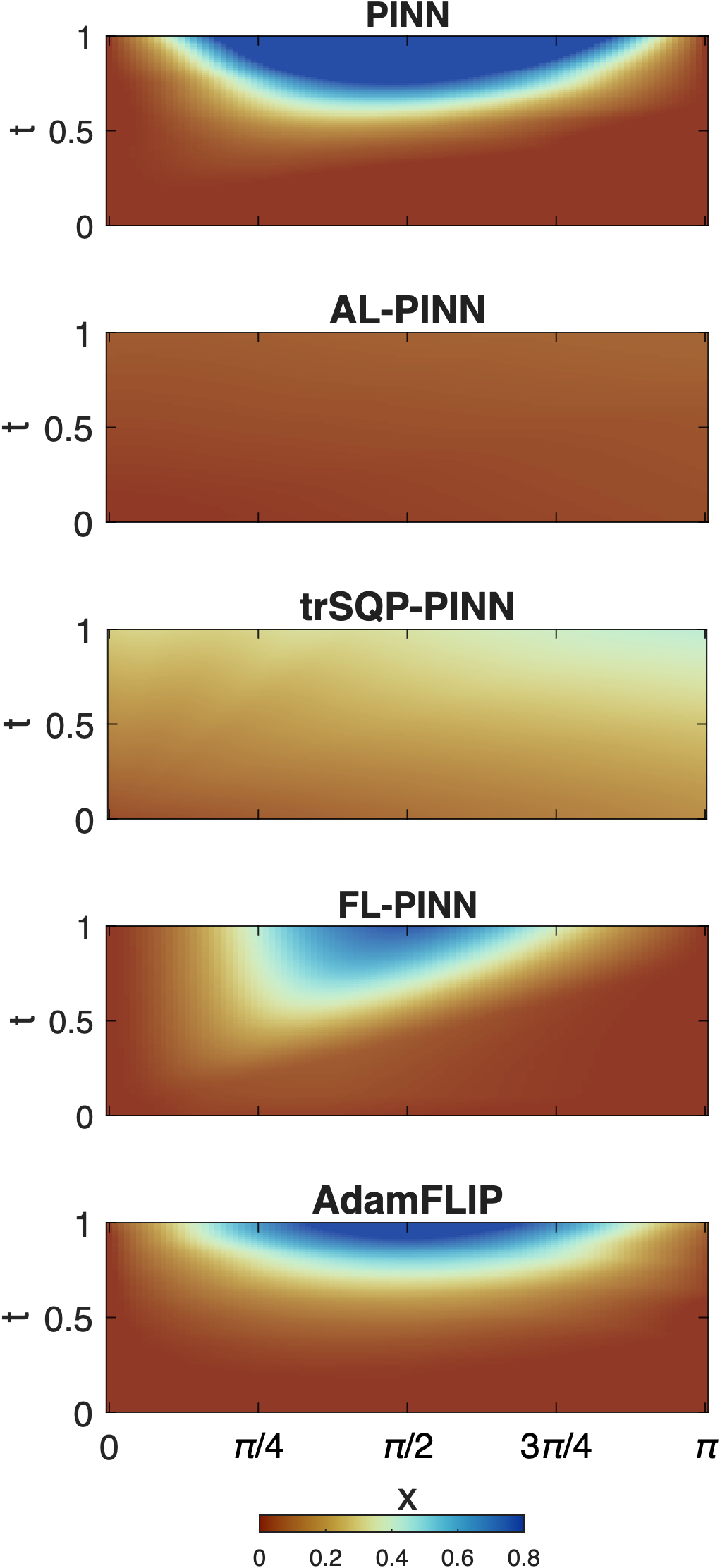}
        \caption{Spatiotemporal fields of different models}
        \label{fig:fpde-inverse}
    \end{subfigure}\hfill
    \begin{subfigure}[b]{0.48\linewidth}
        \centering
        \includegraphics[width=\linewidth]{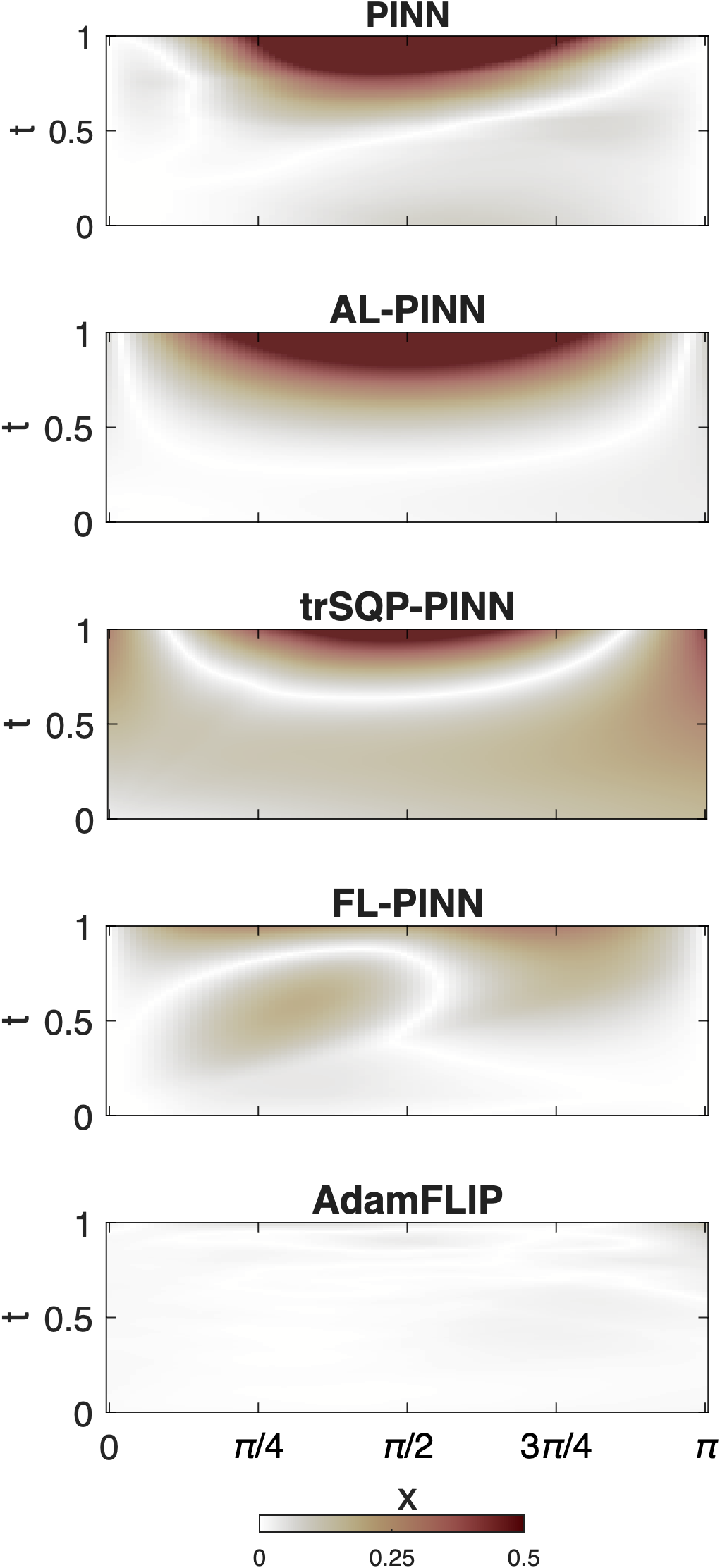}
        \caption{Absolute errors}
        \label{fig:fpde-inverse-error}
    \end{subfigure}
    \caption{Inverse time-fractional mixed diffusion--wave equation (1D) results: spatiotemporal reconstructions and absolute errors. (a) Predicted solution fields $u(x,t)$ obtained by Standard PINN, FL-PINN, trSQP-PINN, and AdamFLIP. (b) Corresponding absolute error maps, showing that AdamFLIP produces the smallest and most uniform errors across the spatiotemporal domain, whereas Standard PINN and FL-PINN exhibit larger localized error bands, especially near the upper-time region.}
    \label{fig:inverse-fpde}
\end{figure}

\subsection{2D Heat Equation}
\label{sec:2d-heat}
We next consider the classical two-dimensional heat equation on the unit square $\Omega = (0,1)^2$ with homogeneous Dirichlet boundary conditions:
\begin{align}
    \label{eq:2d-heat}
&\frac{\partial u(x,y,t) }{\partial t}  - \nu \, \Delta u(x,y,t) \;=\; f(x,y,t), \qquad (x,y)\in\Omega,\; t\in[0,1],\\
&u(x,y,t) = 0,\quad \forall x \in \partial\Omega,\label{eq:2d-heat-2}\\
&u(x,y,0)=u_0(x,y),
\end{align}
where $\Delta$ is the Laplacian operator and $\partial\Omega$ denotes the boundary of the domain. To enable quantitative comparison, we consider the following analytical solution
\begin{equation}
u(x,y,t) = \sin(\pi x)\sin(\pi y)\, \mathrm{e}^{-2\pi^2\kappa t},
\end{equation}
with no forcing, i.e., $f \equiv 0$.

\subsubsection{PINN's settings}
\paragraph{Neural network architecture} 
To ensure a fair and consistent comparison across all different models, we employ an identical deep neural network architecture. Specifically, each model utilizes a fully connected network consisting of $6$ hidden layers with $64$ neurons per layer. During the training phase, each model was trained until its loss dropped below the predefined threshold, $\epsilon$.

\paragraph{Training and testing data}
In the training stage, the initial and boundary condition losses are evaluated using data randomly sampled in space at the initial time $t = 0$ and in time along the spatial boundaries $((x,y)\in\partial\Omega)$. Specifically, we sample $N_{\text{ic}} = 500$ points for the initial condition and $N_{\text{bc}} = 500$ points for the boundary conditions. Additionally, we sample $N_f = 2000$ collocation points within the interior of the spatio-temporal domain to compute the residual loss. During the testing stage, the evaluation data are sampled over the spatiotemporal domain $\Omega \times [0,1]$. The data is generated on a uniform grid with a spatial mesh size of $\Delta x = \Delta y = 0.01$ and a temporal step size of $\Delta t = 0.01$. All samples are drawn uniformly at random from their respective domains. The homogeneous Dirichlet condition $u=0$ is imposed at all boundary points.

\subsubsection{Forward Problem}
In this section, we consider the forward problem setting where the physics loss is fully specified and does not involve any unknown parameters to be inferred which is same as in Section \ref{sec:burgers-forward}.
Table~\ref{tab:heat_results} presents a quantitative comparison of standard PINN (Adam), AL-PINN, trSQP-PINN, FL-PINN, and AdamFLIP for the two-dimensional heat equation. The performance is measured in terms of the relative $L^2$ error and the constraint errors (CE) for the initial condition (IC), boundary condition (BC), and physics loss (PL). Among all models considered, \textit{AdamFLIP} performs best, achieving a relative $L^2$ error that is approximately one order of magnitude smaller than that of the standard PINN. Furthermore, AdamFLIP yields the lowest constraint errors for both the BC and PL components, with its PL constraint error being roughly one-third of that achieved by the standard PINN.

In contrast to the Burgers' equation results, \textit{trSQP-PINN} struggles significantly in this setting, exhibiting a relative $L^2$ error that is more than one order of magnitude higher than the other models. This indicates that the optimization advantages of trSQP-PINN do not necessarily translate to improved accuracy for the heat equation. It is also worth noting that all models yield comparable wall-clock times that implies the improved performance of \textit{AdamFLIP} is achieved without introducing prohibitive computational cost, as its training time remains within approximately one-third of the baseline PINN's overhead.

Figure~\ref{fig:heat-forward} shows the absolute error distributions at the two different time steps. The standard PINN exhibits a noticeable error distribution that intensifies near the boundaries and the center of the domain as the system evolves toward $t = 1.0$. While maintaining the overall physical structure, the FL-PINN shows a concentrated region of higher absolute error near the domain center $(x \approx 0.5, y \approx 0.5)$. Notably, the peak error for this method is approximately one order of magnitude larger than that of the best-performing model. In contrast, the \textit{AdamFLIP} demonstrates the highest precision among all methods. The absolute error maps for AdamFLIP are significantly clearer, with error magnitudes consistently remaining near the lower bound of the scale across both time steps, which is consistent with the results in Table~\ref{tab:heat_results}.

\begin{table}[h]
\centering
\setlength{\tabcolsep}{4pt} 
\caption{Quantitative error comparison for the 2D heat equation (forward problem). We report relative $L^2$ errors, constraint errors (CE) at initial (IC), boundary (BC), and physics (Phy) collocation points, and training wall-clock time (WT).}
\label{tab:heat_results}
\resizebox{\textwidth}{!}{
\begin{tabular}{lcccccc}
\toprule
\textbf{Method} & \textbf{Constraint} & \textbf{Rel. $L_2$} & \textbf{CE (IC)} & \textbf{CE (BC)} & \textbf{CE (Phy)} & \textbf{WT (s)} \\
\midrule
Standard PINN & \xmark 
& $3.43\times10^{-2}$ 
& $1.20\times10^{-4}$ 
& $1.56\times10^{-4}$ 
& $1.82\times10^{-3}$ 
& \textbf{261} \\

AL-PINN & \cmark 
& $3.09\times10^{-2}$ 
& $\mathbf{3.65\times10^{-5}}$ 
& $7.89\times10^{-5}$ 
& $7.41\times10^{-4}$ 
& 364 \\

trSQP-PINN & \cmark 
& $6.60\times10^{-1}$ 
& $1.22\times10^{-1}$ 
& $2.62\times10^{-2}$ 
& $1.07\times10^{-2}$ 
& 1260 \\

FL-PINN & \cmark 
& $9.49\times10^{-3}$ 
& $1.13\times10^{-4}$ 
& $1.70\times10^{-4}$ 
& $7.35\times10^{-4}$ 
& 354 \\

\textbf{AdamFLIP} & \cmark 
& $\mathbf{7.53\times10^{-3}}$ 
& $5.59\times10^{-5}$ 
& $\mathbf{6.35\times10^{-5}}$ 
& $\mathbf{2.00\times10^{-4}}$ 
& 336 \\
\bottomrule
\end{tabular}%
}
\end{table}

\begin{figure}[H]
    \centering
        \includegraphics[width=\linewidth]{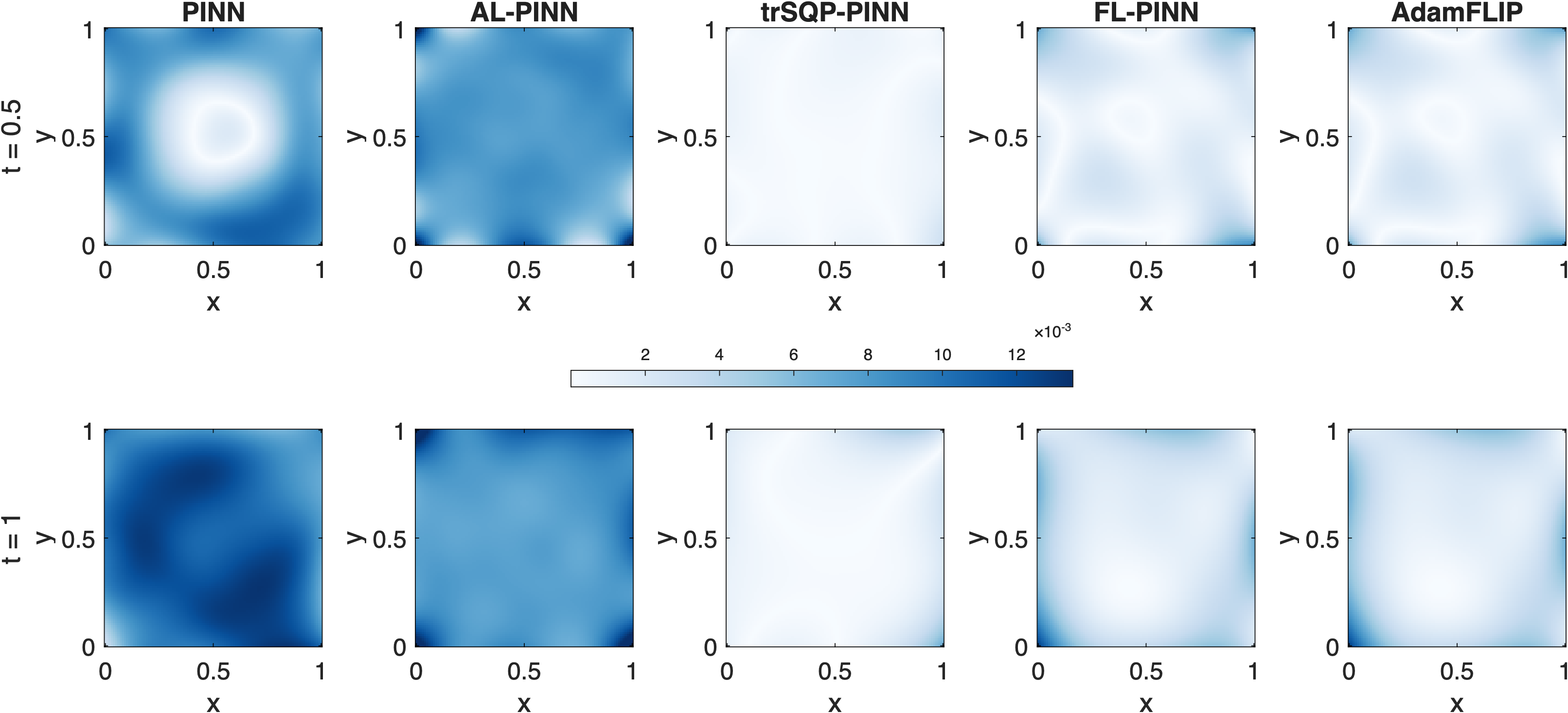}
        \caption{Comparison of spatiotemporal solutions and absolute errors in forward problem for the 2d Heat equation: Standard PINN, AL-PINN, trSQP-PINN, FL-PINN, and AdamFLIP. }
    \label{fig:heat-forward}
\end{figure}

\subsubsection{Inverse Problem}
In the inverse problem setting, we consider the thermal diffusivity $\nu$ in Eq.~\eqref{eq:2d-heat} to be an unknown parameter, hereafter denoted as $\kappa$, which leads to the following formulation:
\begin{align}
    \label{eq:2d-heat-inverse-1}
&\frac{\partial u(x,y,t) }{\partial t}  - \kappa \, \Delta u(x,y,t) \;=\; f(x,y,t), \qquad (x,y)\in\Omega,\; t\in[0,1],\\
&u(x,y,t) = 0,\quad \forall x \in \partial\Omega,\label{eq:2d-heat-inverse-3}\\
&u(x,y,0)=u_0(x,y),
\end{align}
Consistent with the treatement in Section \ref{sec:burgers-inverse}, $\kappa$ is parameterized as learnable scalars, which are co-optimized with the neural network weights by minimizing the loss function same as in the forward problem. In this setup, for all models we initialize the parameter estimates at $\kappa =1$, whereas their ground truth values are $\kappa = 0.1 $.

Table~\ref{tab:inverse_params} shows the learned PDE parameters and their $L^1$ error compared to the ground truth for the Burgers' equation in the inverse problem setting. It is shown that among all models, AdamFLIP demonstrates the highest accuracy by achieving the lowest $L^1$ errors in estimating $\kappa$. Specifically, it reduces the estimation error of $\kappa$ by nearly half of magnitude compared to the next best model, standard PINN.

\begin{table}[h]
\centering
\caption{Parameter estimation results for the two-dimensional heat equation in the inverse problem setting. The exact PDE parameters are $\kappa=0.1$.}
\label{tab:inverse_heat_alpha}
\begin{tabular}{lcc}
\toprule
\textbf{Method} & \textbf{$\hat{\kappa}$} & $|\hat{\kappa}-\kappa|$ \\
\midrule
Standard PINN  & 0.0991 & 0.0009 \\
AL-PINN & 0.1020 & 0.0020 \\
trSQP-PINN & 0.9840 & 0.8840 \\
FL-PINN & 0.0933 & 0.0067 \\
\textbf{AdamFLIP} & \textbf{0.0995} & \textbf{0.0005} \\
\bottomrule
\end{tabular}
\end{table}

\begin{table}[h]
\centering
\setlength{\tabcolsep}{4pt} 
\caption{Quantitative error comparison for the 2D heat equation (inverse problem). We report relative $L^2$ errors, constraint errors (CE) at initial (IC), boundary (BC), physics (Phy), and data loss (DL) collocation points, and training wall-clock time (WT).}
\label{tab:inverse_heat_results}
\resizebox{\textwidth}{!}{
\begin{tabular}{lccccccc}
\toprule
\textbf{Method} & \textbf{Constraint} & \textbf{Rel. $L_2$} & \textbf{CE (IC)} & \textbf{CE (BC)} & \textbf{CE (Phy)} & \textbf{CE (DL)} & \textbf{WT (s)} \\
\midrule
Standard PINN & \xmark 
& $2.90\times10^{-2}$ 
& $3.06\times10^{-5}$ 
& $1.19\times10^{-4}$ 
& $2.79\times10^{-4}$ 
& $5.22\times10^{-5}$ 
& 297 \\

AL-PINN & \cmark 
& $3.98\times10^{-2}$ 
& $1.71\times10^{-4}$ 
& $3.07\times10^{-4}$ 
& $2.29\times10^{-4}$ 
& $9.18\times10^{-5}$ 
& 361 \\

trSQP-PINN & \cmark 
& $6.42\times10^{-1}$ 
& $4.81\times10^{-4}$ 
& $4.57\times10^{-3}$ 
& $3.73\times10^{-3}$ 
& $2.59\times10^{-2}$ 
& 1300 \\

FL-PINN & \cmark 
& $4.06\times10^{-2}$ 
& $1.80\times10^{-5}$ 
& $\mathbf{8.73\times10^{-6}}$ 
& $7.44\times10^{-5}$ 
& $8.81\times10^{-5}$ 
& 360 \\

\textbf{AdamFLIP} & \cmark 
& $\mathbf{1.59\times10^{-2}}$ 
& $\mathbf{1.43\times10^{-5}}$ 
& $1.03\times10^{-5}$ 
& $9.13\times10^{-5}$ 
& $\mathbf{1.51\times10^{-5}}$ 
& 346 \\
\bottomrule
\end{tabular}%
}
\end{table}

Table ~\ref{tab:inverse_heat_results} presents a quantitative error comparison of different models for the two-dimensional heat equation in the inverse problem setting. Beyond the standard constraint errors (CE) for the initial condition (IC), boundary condition (BC), and physics loss (PL), we also report the data loss error (CE(DL)) used to infer the unknown diffusivity $\kappa$. Consistent with the forward problem results, \textit{AdamFLIP} exhibits the most robust performance, achieving a relative $L^2$ error that is approximately half that of the standard PINN. Furthermore, the data loss error (CE(DL)) for AdamFLIP is roughly one-third of the error observed in the standard baseline, indicating a significantly more accurate reconstruction of the underlying field from observed data.

It is also noteworthy that \textit{AdamFLIP} achieves the lowest initial condition error (CE(IC)), which is approximately half the magnitude of the error produced by the standard PINN. In contrast, \textit{trSQP-PINN} continues to struggle in this setting, yielding a relative $L^2$ error and a data loss error that are more than one order of magnitude higher than those of the other models. Regarding computational efficiency, all frameworks yield comparable wall-clock times (WT) though the training time for AdamFLIP is approximately one-fifth higher than the standard PINN.

These results are also confirmed in Figure~\ref{fig:heat-inverse}. While all methods capture the general diffusion dynamics (panel a.), the spatial distributions of absolute error at different time instances (panel. b) reveal that \textit{AdamFLIP} yields a much smaller error distribution in space that effectively mitigating the concentrated interior errors observed in the standard PINN and FL-PINN models.

\begin{figure}[H]
    \centering
        \includegraphics[width=\linewidth]{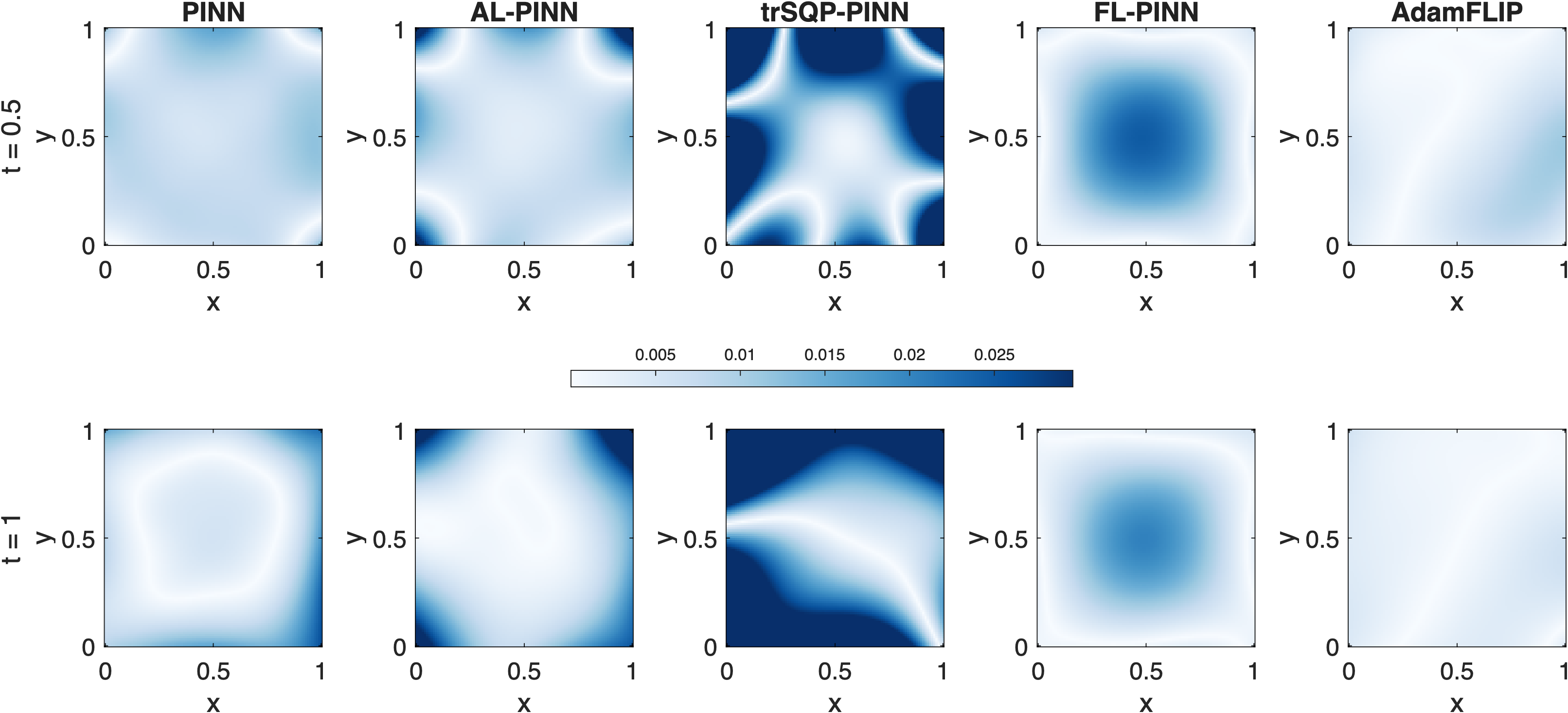}
        \caption{Inverse 2D heat equation results: comparison of absolute errors fields for each method at the same time instants,for PINN, FL-PINN, trSQP-PINN, and AdamFLIP. 
         It shows that AdamFLIP yields the smallest and most uniformly distributed errors across the spatial domain.
        }
    \label{fig:heat-inverse}
\end{figure}

\subsection{2D Incompressible Navier--Stokes}
We evaluate all methods on the 2D incompressible Navier--Stokes equations, a coupled nonlinear system that is substantially more challenging than the scalar PDEs considered above. The governing equations on $\Omega = [0,2\pi]^2$, $t \in [0,T]$ are
\begin{align}
&\frac{\partial u}{\partial t}
+ u \frac{\partial u}{\partial x}
+ v \frac{\partial u}{\partial y}
= - \frac{\partial p}{\partial x}
+ \nu \left(
\frac{\partial^2 u}{\partial x^2}
+ \frac{\partial^2 u}{\partial y^2}
\right), \label{eq:nse-u}\\
&\frac{\partial v}{\partial t}
+ u \frac{\partial v}{\partial x}
+ v \frac{\partial v}{\partial y}
= - \frac{\partial p}{\partial y}
+ \nu \left(
\frac{\partial^2 v}{\partial x^2}
+ \frac{\partial^2 v}{\partial y^2}
\right), \label{eq:nse-v}\\
&\frac{\partial u}{\partial x}
+ \frac{\partial v}{\partial y}
= 0, \label{eq:nse-div}
\end{align}
where $\mathbf{u} = (u,v)$ is the velocity field, $p$ is the pressure, and $\nu = 0.01$. We use the Taylor--Green vortex as the benchmark, which admits the following closed-form analytical solution:
\begin{align}
u(x,y,t) &= -\cos(x)\sin(y)\,e^{-2\nu t}, \label{eq:tgv-u}\\
v(x,y,t) &= \sin(x)\cos(y)\,e^{-2\nu t}, \label{eq:tgv-v}\\
p(x,y,t) &= -\tfrac{1}{4}\left(\cos(2x)+\cos(2y)\right)e^{-4\nu t}. \label{eq:tgv-p}
\end{align}

\subsubsection{Numerical Settings}
\paragraph{Neural network architecture}
To ensure a fair and consistent comparison across all models, we employ an identical fully connected network consisting of $6$ hidden layers with $64$ neurons per layer. The network takes $(x,y,t)$ as input and predicts the velocity-pressure field $(u,v,p)$. During training, each model was trained until its loss dropped below the predefined threshold, $\epsilon$.

\paragraph{Training and testing data}
In the training stage, the initial condition loss is evaluated at $N_{\text{ic}} = 500$ randomly sampled spatial points at $t=0$. The boundary condition loss is evaluated at $N_{\text{bc}} = 500$ points sampled along the spatial boundaries over time. The physics loss, comprising the two momentum equations~\eqref{eq:nse-u}--\eqref{eq:nse-v} and the incompressibility constraint~\eqref{eq:nse-div}, is computed on $N_f = 2000$ collocation points randomly sampled from the interior of the spatiotemporal domain. During the testing stage, the evaluation data are generated on a uniform grid over $\Omega \times [0,T]$ with spatial mesh size $\Delta x = \Delta y = 0.05$ and temporal step size $\Delta t = 0.1$.

We consider the forward problem setting where the viscosity $\nu$ is known and the objective is to accurately approximate the velocity and pressure fields. Table~\ref{tab:nse_forward_results} presents a quantitative comparison of Standard PINN (Adam), AL-PINN, trSQP-PINN, FL-PINN, and AdamFLIP for the Navier--Stokes equations. The performance is measured in terms of the mean relative $L^2$ errors of the velocity components $u$ and $v$, the constraint errors (CE) for the initial condition (IC), boundary condition (BC), and physics residual (Phy), and the training wall-clock time (WT).

Among all models, AdamFLIP achieves the best performance across all accuracy metrics. It attains mean relative $L_2$ errors of $1.18\times10^{-2}$ and $1.06\times10^{-2}$ for the $u$ and $v$ velocity components, respectively---roughly $3\times$ lower than Standard PINN ($3.62\times10^{-2}$ and $3.48\times10^{-2}$) and substantially better than all constrained baselines. AdamFLIP also yields the lowest constraint errors across IC, BC, and physics residual, demonstrating that the feedback linearization mechanism effectively balances multi-component constraint satisfaction in this coupled system.

Notably, trSQP-PINN fails catastrophically on this problem, producing $L_2$ errors on the order of $10^{0}$ and a physics CE of $1.29\times10^{1}$. This is consistent with the sensitivity of trust-region SQP methods to the increased dimensionality and nonlinearity of the Navier--Stokes system. FL-PINN also struggles relative to the scalar PDE benchmarks, with $L_2$ errors ($7.70\times10^{-2}$, $8.13\times10^{-2}$) that are worse than Standard PINN, suggesting that naive feedback linearization without adaptive moment estimation is insufficient for this more complex problem. AL-PINN performs comparably to Standard PINN on the $v$-component but shows degraded IC/BC enforcement.

In terms of computational cost, Standard PINN is the fastest at $155$\,s, while AdamFLIP requires $335$\,s---roughly $2\times$ the baseline, which is a larger overhead than observed on the scalar PDEs. This additional cost is attributable to the Jacobian computations required for the multi-output, multi-constraint structure of the Navier--Stokes system. Nevertheless, the accuracy improvement is substantial.

\begin{table}[h]
\centering
\setlength{\tabcolsep}{4pt}
\caption{Quantitative error comparison for the 2D Navier--Stokes equations (forward problem). We report mean relative $L^2$ errors for velocity components $u$ and $v$, constraint errors (CE) at initial (IC), boundary (BC), and physics (Phy) collocation points.}
\label{tab:nse_forward_results}
\resizebox{\textwidth}{!}{%
\begin{tabular}{lcccccc}
\toprule
\textbf{Method} & \textbf{Constraint} & \textbf{Mean $L_2(u)$} & \textbf{Mean $L_2(v)$} & \textbf{CE (IC)} & \textbf{CE (BC)} & \textbf{CE (Phy)} \\
\midrule

Standard PINN & \xmark 
& $3.62\times10^{-2}$ 
& $3.48\times10^{-2}$ 
& $2.15\times10^{-5}$ 
& $1.02\times10^{-4}$ 
& $6.11\times10^{-4}$ \\

AL-PINN & \cmark
& $5.19\times10^{-2}$
& $3.69\times10^{-2}$
& $1.10\times10^{-3}$
& $8.13 \times10^{-4}$
& $8.56\times10^{-4}$ \\

trSQP-PINN & \cmark
& $1.96\times10^{0}$
& $9.68\times10^{-1}$
& $7.82\times10^{-4}$
& $3.56\times10^{-4}$
& $1.29\times10^{1}$ \\

FL-PINN & \cmark 
& $7.70\times10^{-2}$ 
& $8.13\times10^{-2}$ 
& $1.98\times10^{-3}$ 
& $2.72\times10^{-3}$ 
& $1.08\times10^{-2}$ \\

\textbf{AdamFLIP} & \cmark 
& $\mathbf{1.18\times10^{-2}}$ 
& $\mathbf{1.06\times10^{-2}}$ 
& $\mathbf{1.26\times10^{-5}}$ 
& $\mathbf{4.86\times10^{-5}}$ 
& $\mathbf{2.53\times10^{-4}}$ \\

\bottomrule
\end{tabular}%
}
\end{table}

Figure~\ref{fig:contour-nse} presents the pointwise absolute error maps $|u - \hat{u}|$ and $|v - \hat{v}|$ at $t=1$ for all five methods. Standard PINN exhibits moderate, spatially distributed errors that are roughly uniform across the domain. AL-PINN shows a similar error pattern with slightly larger magnitudes. trSQP-PINN produces large-scale errors spanning the entire domain, consistent with its failure to converge to a physically meaningful solution. FL-PINN displays concentrated error regions, particularly near the vortex cores where the velocity gradients are steepest. In contrast, AdamFLIP achieves the smallest and most spatially uniform errors for both velocity components, with residual errors primarily confined to narrow bands near regions of maximum shear. These qualitative observations are fully consistent with the quantitative results reported in Table~\ref{tab:nse_forward_results}.

\begin{figure}[H]
    \centering
    \includegraphics[width=\linewidth]{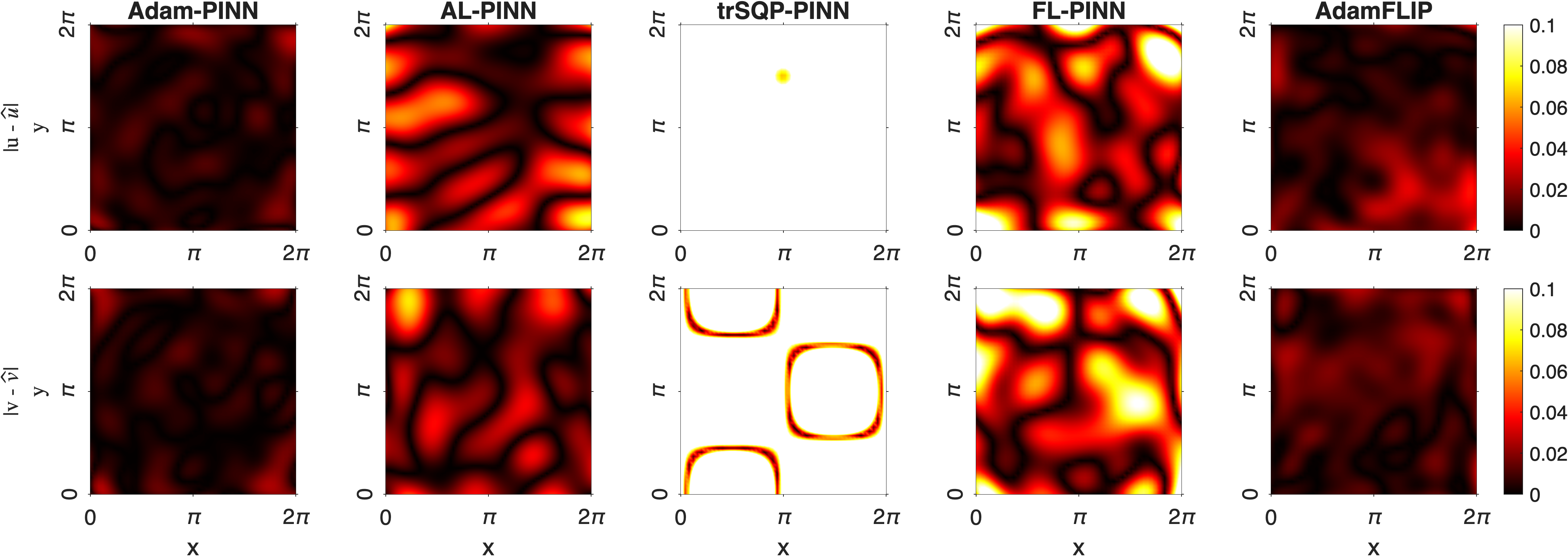}
    \caption{Pointwise absolute error $|u - \hat{u}|$ (top row) and $|v - \hat{v}|$ (bottom row) at $t=1$ for the 2D Navier--Stokes equations. From left to right: Standard PINN, AL-PINN, trSQP-PINN, FL-PINN, and AdamFLIP. AdamFLIP produces the smallest and most spatially uniform errors across both velocity components.}
    \label{fig:contour-nse}
\end{figure}
\section{Further Ablations}
Table~\ref{tab:beta_ablation} studies the effect of the momentum parameters
$(\beta_1,\beta_2)$ under fixed $K_p=1000$ and $K_i=0.01$.
Overall, introducing stronger first- and second-moment smoothing significantly
reduces the relative $L_2$ error compared with the weak-momentum baseline.
The setting $(\beta_1,\beta_2)=(0.99,0.9999)$ achieves the lowest mean relative
$L_2$ error, indicating that aggressive momentum smoothing can improve the final
solution accuracy. However, this configuration may be more sensitive to random
initialization and training dynamics. In comparison, $(\beta_1,\beta_2)=(0.95,0.999)$
also yields a very low relative $L_2$ error and provides a competitive balance
between accuracy and stability. Therefore, we use $(\beta_1,\beta_2)=(0.95,0.999)$
as the default configuration in our experiments.

\begin{table}[H]
\centering
\scriptsize
\setlength{\tabcolsep}{6pt}
\renewcommand{\arraystretch}{1.08}
\caption{
Ablation study on the momentum parameters $(\beta_1,\beta_2)$ with fixed
$K_p=1000$ and $K_i=0.01$.
We report the relative $L_2$ error averaged over three random seeds.
Lower values are better. The best result is highlighted in bold.
}
\label{tab:beta_ablation}
\resizebox{0.65\linewidth}{!}{
\begin{tabular}{ccc}
\toprule
$\boldsymbol{\beta_1}$ & $\boldsymbol{\beta_2}$
& \textbf{Rel. $L_2$ Error} \\
\midrule
$0.00$ & $0.9000$ 
& $2.15{\times}10^{-1} \pm 3.34{\times}10^{-2}$ \\

$0.50$ & $0.9900$ 
& $1.00{\times}10^{-1} \pm 1.28{\times}10^{-2}$ \\

$0.90$ & $0.9900$ 
& $1.33{\times}10^{-1} \pm 1.56{\times}10^{-2}$ \\

$0.90$ & $0.9990$ 
& $1.10{\times}10^{-1} \pm 1.89{\times}10^{-2}$ \\

$0.95$ & $0.9990$ 
& $8.48{\times}10^{-3} \pm 5.01{\times}10^{-2}$ \\

$0.99$ & $0.9999$ 
& $\mathbf{7.53{\times}10^{-3} \pm 6.99{\times}10^{-3}}$ \\
\bottomrule
\end{tabular}
}
\end{table}

Table~\ref{tab:kp_constraint_ablation} studies the effect of the proportional
penalty gain $K_p$ on the objective loss and constraint satisfaction. Here, the
physics loss is treated as the optimization objective, while the initial
condition and boundary condition losses are treated as constraints. As $K_p$
increases, both constraint losses decrease substantially, indicating that a
larger proportional gain improves enforcement of the initial and boundary
conditions. However, this comes at the cost of a larger physics loss, revealing
a clear trade-off between minimizing the PDE residual objective and satisfying
the constraints. Small values such as $K_p=10$ achieve the lowest physics loss
but fail to enforce the constraints, especially the initial condition loss. In
contrast, very large values such as $K_p=5000$ provide strong boundary
constraint satisfaction but significantly increase the physics loss. The setting
$K_p=1000$ provides a balanced trade-off, achieving much lower IC and BC losses
than small-$K_p$ settings while avoiding the larger objective degradation
observed at $K_p=5000$. Therefore, we select $K_p=1000$ as the default
configuration.

\begin{table}[t]
\centering
\scriptsize
\setlength{\tabcolsep}{5.5pt}
\renewcommand{\arraystretch}{1.08}
\caption{
Ablation study on the proportional penalty gain $K_p$.
We report the physics loss as the objective and the initial condition (IC) and
boundary condition (BC) losses as constraints. Results are averaged over three
random seeds and reported as mean $\pm$ standard deviation.
Lower values are better. The selected default configuration is highlighted in bold.
}
\label{tab:kp_constraint_ablation}
\resizebox{\linewidth}{!}{
\begin{tabular}{cccc}
\toprule
\multirow{2}{*}{$\boldsymbol{K_p}$}
& \textbf{Objective}
& \multicolumn{2}{c}{\textbf{Constraints}} \\
\cmidrule(lr){2-2}
\cmidrule(lr){3-4}
& \textbf{Physics Loss}
& \textbf{IC Loss}
& \textbf{BC Loss} \\
\midrule
$10$
& $3.83{\times}10^{-4} \pm 5.66{\times}10^{-5}$
& $9.83{\times}10^{-1} \pm 1.17{\times}10^{-1}$
& $8.48{\times}10^{-2} \pm 4.87{\times}10^{-2}$ \\

$50$
& $4.71{\times}10^{-3} \pm 5.96{\times}10^{-4}$
& $3.88{\times}10^{-2} \pm 5.31{\times}10^{-3}$
& $3.55{\times}10^{-2} \pm 7.72{\times}10^{-3}$ \\

$100$
& $7.34{\times}10^{-3} \pm 4.68{\times}10^{-4}$
& $2.54{\times}10^{-2} \pm 1.33{\times}10^{-2}$
& $2.95{\times}10^{-2} \pm 1.48{\times}10^{-2}$ \\

$200$
& $1.01{\times}10^{-2} \pm 1.07{\times}10^{-3}$
& $3.55{\times}10^{-2} \pm 3.08{\times}10^{-2}$
& $2.68{\times}10^{-2} \pm 1.49{\times}10^{-2}$ \\

$500$
& $1.12{\times}10^{-2} \pm 6.52{\times}10^{-4}$
& $2.27{\times}10^{-2} \pm 1.76{\times}10^{-2}$
& $1.89{\times}10^{-2} \pm 1.29{\times}10^{-2}$ \\

$\mathbf{1000}$
& $\mathbf{1.30{\times}10^{-2} \pm 1.02{\times}10^{-3}}$
& $\mathbf{1.26{\times}10^{-2} \pm 4.99{\times}10^{-3}}$
& $\mathbf{1.31{\times}10^{-2} \pm 4.80{\times}10^{-3}}$ \\

$2000$
& $1.43{\times}10^{-2} \pm 1.56{\times}10^{-3}$
& $8.16{\times}10^{-3} \pm 6.20{\times}10^{-4}$
& $9.25{\times}10^{-3} \pm 1.50{\times}10^{-3}$ \\

$5000$
& $2.25{\times}10^{-2} \pm 3.38{\times}10^{-3}$
& $1.10{\times}10^{-2} \pm 4.69{\times}10^{-3}$
& $9.28{\times}10^{-3} \pm 4.49{\times}10^{-3}$ \\
\bottomrule
\end{tabular}
}
\end{table}

\section{Proof of Theorem~\ref{thm:main}}
\label{sec:proof-main}

We prove Theorem~\ref{thm:main} using the exact penalty merit function
\begin{align*}
    \Psi_\rho(\theta):=f(\theta)+\rho\|h(\theta)\|_1.
\end{align*}
Throughout the proof, all constants are uniform over $\theta\in\mathcal C$ and over the admissible diagonal metrics generated by Algorithm~\ref{alg:metric-adamflip}.

\begin{lemma}[Uniform metric bounds and finite metric variation]
\label{lem:metric-bounds-variation}
Under Assumption~\ref{assump:smooth-compact}, the diagonal metrics $D_t$ satisfy
\begin{align*}
    \frac{1}{G+\delta}I\preceq D_t\preceq \frac{1}{\delta}I\qquad \forall t.
\end{align*}
Moreover, if $\Delta_t:=\|D_{t+1}-D_t\|$, then
\begin{align*}
    \sum_{t=1}^{\infty}\Delta_t\le \frac{n}{\delta},
\end{align*}
and therefore
\begin{align*}
    \sum_{t=1}^{\infty}\eta_t\Delta_t\le \frac{\eta_1 n}{\delta}.
\end{align*}
\end{lemma}

\begin{proof}
Since $v_t=\beta_2v_{t-1}+(1-\beta_2)g_t^{\circ2}$ and $v_0=0$, the bias-corrected raw second moment $\tilde v_t=v_t/(1-\beta_2^t)$ is a coordinatewise convex combination of $g_1^{\circ2},\dots,g_t^{\circ2}$. Hence, by Assumption~\ref{assump:smooth-compact}, for every coordinate $j$,
\begin{align*}
    0\le \tilde v_{t,j}\le G^2.
\end{align*}
The monotone second-moment estimate is defined by
\begin{align*}
    \bar v_t=\max_{\rm elem}\{\bar v_{t-1},\tilde v_t\}.
\end{align*}
Thus $0\le \bar v_{t,j}\le G^2$ for all $t,j$. Since
\begin{align*}
    D_t=\diag\left(\frac{1}{\sqrt{\bar v_{t-1}}+\delta}\right),
\end{align*}
we have, for every coordinate $j$,
\begin{align*}
    \frac{1}{G+\delta}\le \frac{1}{\sqrt{\bar v_{t-1,j}}+\delta}\le \frac{1}{\delta}.
\end{align*}
This proves the uniform metric bounds.

It remains to prove finite variation. For each coordinate $j$, define
\begin{align*}
    d_{t,j}:=\frac{1}{\sqrt{\bar v_{t-1,j}}+\delta}.
\end{align*}
Since $\bar v_{t,j}$ is coordinatewise nondecreasing in $t$, the sequence $d_{t,j}$ is nonincreasing in $t$. Moreover, $0<d_{t,j}\le 1/\delta$. Therefore,
\begin{align*}
    \sum_{t=1}^{\infty}|d_{t+1,j}-d_{t,j}|=\sum_{t=1}^{\infty}(d_{t,j}-d_{t+1,j})\le d_{1,j}\le \frac{1}{\delta}.
\end{align*}
Since $D_t$ is diagonal,
\begin{align*}
    \|D_{t+1}-D_t\|\le \|D_{t+1}-D_t\|_F\le \sum_{j=1}^n |d_{t+1,j}-d_{t,j}|.
\end{align*}
Summing over $t$ gives
\begin{align*}
    \sum_{t=1}^{\infty}\Delta_t
    &\le \sum_{j=1}^n\sum_{t=1}^{\infty}|d_{t+1,j}-d_{t,j}| \\
    &\le \frac{n}{\delta}.
\end{align*}
Finally, since $\eta_t\le \eta_1$,
\begin{align*}
    \sum_{t=1}^{\infty}\eta_t\Delta_t\le \eta_1\sum_{t=1}^{\infty}\Delta_t\le \frac{\eta_1 n}{\delta}.
\end{align*}
This proves the lemma.
\end{proof}

\begin{lemma}[Momentum tracking error]
\label{lem:momentum-error-amsgrad}
Let $e_t:=\hat m_t-g_t$. Then there exists $C_e>0$ such that, for every $T\ge1$,
\begin{align*}
    \sum_{t=1}^T\eta_t\|e_t\|\le C_e\left(\sum_{t=1}^T\eta_t^2+\sum_{t=1}^T\eta_t\Delta_t\right).
\end{align*}
Consequently, by Lemma~\ref{lem:metric-bounds-variation},
\begin{align*}
    \sum_{t=1}^T\eta_t\|e_t\|\le C_e\left(\sum_{t=1}^T\eta_t^2+\frac{\eta_1 n}{\delta}\right).
\end{align*}
\end{lemma}

\begin{proof}
The bias-corrected first moment can be written as
\begin{align*}
    \hat m_t=\sum_{i=1}^t a_{t,i}g_i,\qquad a_{t,i}:=\frac{(1-\beta_1)\beta_1^{t-i}}{1-\beta_1^t}.
\end{align*}
The weights satisfy $a_{t,i}\ge0$ and $\sum_{i=1}^t a_{t,i}=1$. Hence
\begin{align*}
    e_t=\hat m_t-g_t=\sum_{i=1}^t a_{t,i}(g_i-g_t).
\end{align*}

By Assumption~\ref{assump:smooth-compact} and Lemma~\ref{lem:metric-bounds-variation}, the map $(\theta,D)\mapsto g_D(\theta)$ is Lipschitz on the compact admissible set
\begin{align*}
    \mathcal C\times \left\{D:\frac{1}{G+\delta}I\preceq D\preceq \frac{1}{\delta}I\right\}.
\end{align*}
Let $L_g>0$ be a corresponding Lipschitz constant. Then, for $i\le t$,
\begin{align*}
    \|g_i-g_t\|\le L_g\big(\|\theta_i-\theta_t\|+\|D_i-D_t\|\big).
\end{align*}
Since $\hat m_t$ is a convex combination of previous $g_i$'s and $\|g_i\|\le G$, we have $\|\hat m_t\|\le G$. Together with $\|D_t\|\le 1/\delta$, this gives
\begin{align*}
    \|\theta_{t+1}-\theta_t\|=\eta_t\|D_t\hat m_t\|\le \frac{G}{\delta}\eta_t.
\end{align*}
Therefore, for $i\le t$,
\begin{align*}
    \|\theta_i-\theta_t\|\le \frac{G}{\delta}\sum_{k=i}^{t-1}\eta_k,\qquad \|D_i-D_t\|\le \sum_{k=i}^{t-1}\Delta_k.
\end{align*}
Then
\begin{align*}
    \|g_i-g_t\|\le L_g\frac{G}{\delta}\sum_{k=i}^{t-1}(\eta_k+\Delta_k).
\end{align*}
Thus
\begin{align*}
    \|e_t\|
    &\le L_g\frac{G}{\delta}\sum_{i=1}^t a_{t,i}\sum_{k=i}^{t-1}(\eta_k+\Delta_k) \\
    &=L_g\frac{G}{\delta}\sum_{k=1}^{t-1}(\eta_k+\Delta_k)\sum_{i=1}^k a_{t,i}.
\end{align*}
Using the explicit form of $a_{t,i}$,
\begin{align*}
    \sum_{i=1}^k a_{t,i}=\frac{\beta_1^{t-k}-\beta_1^t}{1-\beta_1^t}\le \frac{\beta_1^{t-k}}{1-\beta_1}.
\end{align*}
Therefore,
\begin{align*}
    \|e_t\|\le \frac{L_gG}{(1-\beta_1)\delta}\sum_{k=1}^{t-1}\beta_1^{t-k}(\eta_k+\Delta_k).
\end{align*}
Multiplying by $\eta_t$ and summing from $t=1$ to $T$ yields
\begin{align*}
    \sum_{t=1}^T\eta_t\|e_t\|
    &\le \frac{L_gG}{(1-\beta_1)\delta}\sum_{t=1}^T\eta_t\sum_{k=1}^{t-1}\beta_1^{t-k}(\eta_k+\Delta_k) \\
    &=\frac{L_gG}{(1-\beta_1)\delta}\sum_{k=1}^{T-1}(\eta_k+\Delta_k)\sum_{t=k+1}^T\eta_t\beta_1^{t-k}.
\end{align*}
Since $\eta_t$ is nonincreasing, for $t\ge k+1$ we have $\eta_t\le \eta_k$. Hence
\begin{align*}
    \sum_{t=k+1}^T\eta_t\beta_1^{t-k}\le \eta_k\sum_{\ell=1}^{\infty}\beta_1^\ell=\eta_k\frac{\beta_1}{1-\beta_1}.
\end{align*}
Therefore,
\begin{align*}
    \sum_{t=1}^T\eta_t\|e_t\|
    &\le \frac{L_gG\beta_1}{(1-\beta_1)^2\delta}\sum_{k=1}^{T-1}\eta_k(\eta_k+\Delta_k) \\
    &\le C_e\left(\sum_{t=1}^T\eta_t^2+\sum_{t=1}^T\eta_t\Delta_t\right),
\end{align*}
where one may take
\begin{align*}
    C_e:=\frac{L_gG\beta_1}{(1-\beta_1)^2\delta}.
\end{align*}
The second claim follows immediately from Lemma~\ref{lem:metric-bounds-variation}. This proves the lemma.
\end{proof}

\begin{proof}[Proof of Theorem~\ref{thm:main}]
By Lemma~\ref{lem:metric-bounds-variation}, the matrices $D_t$ are uniformly positive definite and uniformly bounded. Since $J_h(\theta)$ has full row rank on $\mathcal C$, the matrices
\begin{align*}
    M_t=J_h(\theta_t)D_tJ_h(\theta_t)^\top
\end{align*}
are uniformly nonsingular. Therefore, $\lambda_t^\dagger$, $r_t$, and $g_t$ are all uniformly bounded on the compact admissible set. In particular, define
\begin{align*}
    G_f:=\sup_{\theta\in\mathcal C}\|\nabla f(\theta)\|,\qquad J_1:=\sup_{\theta\in\mathcal C}\|J_h(\theta)\|_{2\to1},\qquad \Lambda:=\sup_{t\ge1}\|\lambda_t^\dagger\|_\infty.
\end{align*}
All three constants are finite.

Let $L_f>0$ be a Lipschitz constant of $\nabla f$ on the neighborhood specified in Assumption~\ref{assump:smooth-compact}. Let $L_{h,1}>0$ satisfy
\begin{align*}
    \|h(\theta+s)-h(\theta)-J_h(\theta)s\|_1\le \frac{L_{h,1}}{2}\|s\|^2
\end{align*}
whenever the line segment between $\theta$ and $\theta+s$ lies in that neighborhood.

Let $e_t:=\hat m_t-g_t$. The update can be written as
\begin{align*}
    \theta_{t+1}=\theta_t-\eta_tD_tg_t-\eta_tD_te_t.
\end{align*}

We first estimate the objective decrease. By smoothness of $f$,
\begin{align*}
    f(\theta_{t+1})
    &\le f(\theta_t)+\left\langle \nabla f(\theta_t),\theta_{t+1}-\theta_t\right\rangle+\frac{L_f}{2}\|\theta_{t+1}-\theta_t\|^2 \\
    &= f(\theta_t)-\eta_t\left\langle \nabla f(\theta_t),D_tg_t\right\rangle-\eta_t\left\langle \nabla f(\theta_t),D_te_t\right\rangle+\frac{L_f}{2}\eta_t^2\|D_t\hat m_t\|^2.
\end{align*}
Using $\|\nabla f(\theta_t)\|\le G_f$, $\|D_t\|\le 1/\delta$, and $\|\hat m_t\|\le G$, we obtain
\begin{align}
    f(\theta_{t+1})\le f(\theta_t)-\eta_t\left\langle \nabla f(\theta_t),D_tg_t\right\rangle+\frac{\eta_tG_f}{\delta}\|e_t\|+\frac{L_fG^2}{2\delta^2}\eta_t^2.
    \label{eq:obj-descent-ams}
\end{align}

We now simplify the leading inner product. By definition,
\begin{align*}
    r_t=\nabla f(\theta_t)+J_h(\theta_t)^\top\lambda_t^\dagger,
\end{align*}
so
\begin{align*}
    \nabla f(\theta_t)=r_t-J_h(\theta_t)^\top\lambda_t^\dagger.
\end{align*}
Therefore,
\begin{align*}
    \left\langle \nabla f(\theta_t),D_tg_t\right\rangle
    &=\left\langle r_t,D_tg_t\right\rangle-\left\langle J_h(\theta_t)^\top\lambda_t^\dagger,D_tg_t\right\rangle \\
    &=\left\langle r_t,D_tg_t\right\rangle-\left\langle \lambda_t^\dagger,J_h(\theta_t)D_tg_t\right\rangle.
\end{align*}
We have $J_h(\theta_t)D_tr_t=0$ by construction. Moreover,
\begin{align*}
    g_t=r_t+\kappa J_h(\theta_t)^\top M_t^{-1}h(\theta_t),
\end{align*}
and hence
\begin{align*}
    \left\langle r_t,D_tg_t\right\rangle
    &=\left\langle r_t,D_tr_t\right\rangle+\kappa\left\langle r_t,D_tJ_h(\theta_t)^\top M_t^{-1}h(\theta_t)\right\rangle \\
    &=\left\langle r_t,D_tr_t\right\rangle+\kappa\left\langle J_h(\theta_t)D_tr_t,M_t^{-1}h(\theta_t)\right\rangle \\
    &=\left\langle r_t,D_tr_t\right\rangle.
\end{align*}
Also, the metric feedback-linearization identity gives
\begin{align*}
    J_h(\theta_t)D_tg_t=\kappa h(\theta_t).
\end{align*}
Thus
\begin{align*}
    \left\langle \nabla f(\theta_t),D_tg_t\right\rangle=\left\langle r_t,D_tr_t\right\rangle-\kappa{\lambda_t^\dagger}^\top h(\theta_t).
\end{align*}
Using $D_t\succeq \frac{1}{G+\delta}I$ and $\|\lambda_t^\dagger\|_\infty\le \Lambda$, we obtain
\begin{align}
    -\left\langle \nabla f(\theta_t),D_tg_t\right\rangle\le -\frac{1}{G+\delta}\|r_t\|^2+\kappa\Lambda\|h(\theta_t)\|_1.
    \label{eq:inner-bound-ams}
\end{align}
Combining \eqref{eq:obj-descent-ams} and \eqref{eq:inner-bound-ams} yields
\begin{align}
    f(\theta_{t+1})\le f(\theta_t)-\frac{\eta_t}{G+\delta}\|r_t\|^2+\eta_t\kappa\Lambda\|h(\theta_t)\|_1+\frac{\eta_tG_f}{\delta}\|e_t\|+\frac{L_fG^2}{2\delta^2}\eta_t^2.
    \label{eq:f-descent-final-ams}
\end{align}

Next, we estimate the constraint violation. By Taylor's theorem,
\begin{align*}
    h(\theta_{t+1})
    &=h(\theta_t)+J_h(\theta_t)(\theta_{t+1}-\theta_t)+\mathcal E_t \\
    &=h(\theta_t)-\eta_tJ_h(\theta_t)D_tg_t-\eta_tJ_h(\theta_t)D_te_t+\mathcal E_t \\
    &=(1-\kappa\eta_t)h(\theta_t)-\eta_tJ_h(\theta_t)D_te_t+\mathcal E_t,
\end{align*}
where
\begin{align*}
    \|\mathcal E_t\|_1\le \frac{L_{h,1}}{2}\|\theta_{t+1}-\theta_t\|^2.
\end{align*}
Since $\|\theta_{t+1}-\theta_t\|\le \eta_t\|D_t\hat m_t\|\le \eta_tG/\delta$, we have
\begin{align*}
    \|\mathcal E_t\|_1\le \frac{L_{h,1}G^2}{2\delta^2}\eta_t^2.
\end{align*}
Using $\eta_t\le1/\kappa$, we have $1-\kappa\eta_t\ge0$. Therefore,
\begin{align}
    \|h(\theta_{t+1})\|_1
    &\le (1-\kappa\eta_t)\|h(\theta_t)\|_1+\eta_t\|J_h(\theta_t)D_te_t\|_1+\frac{L_{h,1}G^2}{2\delta^2}\eta_t^2 \notag\\
    &\le (1-\kappa\eta_t)\|h(\theta_t)\|_1+\frac{\eta_tJ_1}{\delta}\|e_t\|+\frac{L_{h,1}G^2}{2\delta^2}\eta_t^2.
    \label{eq:h-descent-final-ams}
\end{align}

Multiplying \eqref{eq:h-descent-final-ams} by $\rho$ and adding it to \eqref{eq:f-descent-final-ams}, we obtain
\begin{align}
    \Psi_\rho(\theta_{t+1})
    &\le \Psi_\rho(\theta_t)-\frac{\eta_t}{G+\delta}\|r_t\|^2-\eta_t\kappa(\rho-\Lambda)\|h(\theta_t)\|_1 \notag\\
    &\quad+\frac{\eta_t(G_f+\rho J_1)}{\delta}\|e_t\|+\frac{(L_f+\rho L_{h,1})G^2}{2\delta^2}\eta_t^2.
    \label{eq:psi-descent-ams}
\end{align}
Choose $\rho>\Lambda$ and define
\begin{align*}
    c_*:=\min\left\{\frac{1}{G+\delta},\kappa(\rho-\Lambda)\right\},\qquad B_e:=\frac{G_f+\rho J_1}{\delta},\qquad B_\eta:=\frac{(L_f+\rho L_{h,1})G^2}{2\delta^2}.
\end{align*}
Then \eqref{eq:psi-descent-ams} implies
\begin{align}
    \Psi_\rho(\theta_{t+1})\le \Psi_\rho(\theta_t)-c_*\eta_t\left(\|r_t\|^2+\|h(\theta_t)\|_1\right)+B_e\eta_t\|e_t\|+B_\eta\eta_t^2.
    \label{eq:psi-descent-compact-ams}
\end{align}

Summing \eqref{eq:psi-descent-compact-ams} from $t=1$ to $T$, we get
\begin{align*}
    c_*\sum_{t=1}^T\eta_t\left(\|r_t\|^2+\|h(\theta_t)\|_1\right)
    &\le \Psi_\rho(\theta_1)-\Psi_\rho(\theta_{T+1})+B_e\sum_{t=1}^T\eta_t\|e_t\|+B_\eta\sum_{t=1}^T\eta_t^2.
\end{align*}
Since $\Psi_\rho$ is continuous on $\mathcal C$, it is bounded below. Let
\begin{align*}
    \underline\Psi:=\inf_{\theta\in\mathcal C}\Psi_\rho(\theta)>-\infty.
\end{align*}
Using Lemma~\ref{lem:momentum-error-amsgrad}, we obtain
\begin{align*}
    \sum_{t=1}^T\eta_t\|e_t\|\le C_e\left(\sum_{t=1}^T\eta_t^2+\frac{\eta_1n}{\delta}\right).
\end{align*}
Therefore,
\begin{align*}
    c_*\sum_{t=1}^T\eta_t\left(\|r_t\|^2+\|h(\theta_t)\|_1\right)
    &\le \Psi_\rho(\theta_1)-\underline\Psi+B_eC_e\left(\sum_{t=1}^T\eta_t^2+\frac{\eta_1n}{\delta}\right)+B_\eta\sum_{t=1}^T\eta_t^2.
\end{align*}
Thus
\begin{align}
    \sum_{t=1}^T\eta_t\left(\|r_t\|^2+\|h(\theta_t)\|_1\right)\le B_0+B_1\sum_{t=1}^T\eta_t^2,
    \label{eq:finite-time-main}
\end{align}
where
\begin{align*}
    B_0:=\frac{\Psi_\rho(\theta_1)-\underline\Psi+B_eC_e\eta_1n/\delta}{c_*},\qquad B_1:=\frac{B_eC_e+B_\eta}{c_*}.
\end{align*}
This proves the first claim.

Finally, let $S_T:=\sum_{t=1}^T\eta_t$. Since
\begin{align*}
    S_T\min_{1\le t\le T}\left(\|r_t\|^2+\|h(\theta_t)\|_1\right)\le \sum_{t=1}^T\eta_t\left(\|r_t\|^2+\|h(\theta_t)\|_1\right),
\end{align*}
dividing \eqref{eq:finite-time-main} by $S_T$ gives
\begin{align*}
    \min_{1\le t\le T}\left(\|r_t\|^2+\|h(\theta_t)\|_1\right)\le \frac{B_0+B_1\sum_{t=1}^T\eta_t^2}{S_T}.
\end{align*}
By Assumption~\ref{assump:stepsize}, $\sum_t\eta_t^2<\infty$. Hence
\begin{align*}
    \min_{1\le t\le T}\left(\|r_t\|^2+\|h(\theta_t)\|_1\right)=\mathcal O\left(\frac{1}{S_T}\right).
\end{align*}
This completes the proof.
\end{proof}

\newpage

\end{document}